\documentclass[11pt]{article}
\usepackage{enumerate}
\usepackage[OT1]{fontenc}
\usepackage{amsmath,amssymb}
\usepackage{natbib}
\usepackage[usenames]{color}

\usepackage{smile}
\usepackage{natbib}
\usepackage{multirow}
\usepackage[colorlinks,
            linkcolor=red,
            anchorcolor=blue,
            citecolor=blue
            ]{hyperref}

\def\tr{\mathop{\text{tr}}\kern.2ex}

\def\E{{\mathbb E}}

\long\def\comment#1{}

\def\vec{\mathop{\text{vec}}}
\def\tr{\mathop{\text{Tr}}}

\def\cK{{\mathcal{K}}}
\def\cX{{\mathcal{X}}}

\def\cN{{\mathcal{N}}}
\def\cT{{\mathcal{T}}}
\def\cB{{\mathcal{B}}}
\def\tr{{\text{Tr}}}
\def\tf{{\text{F}}}
\def\dr{\displaystyle \rm}

\newcommand{\bel}{\begin{eqnarray}\label}
\newcommand{\eel}{\end{eqnarray}}
\newcommand{\bes}{\begin{eqnarray*}}
\newcommand{\ees}{\end{eqnarray*}}

\def\real{{\mathbb{R}}}

\usepackage{mathrsfs}

\usepackage{fullpage}
\def\sep{\;\big|\;}

\usepackage{hyperref}
\usepackage{breakcites}
\usepackage{microtype}
\def\##1\#{\begin{align}#1\end{align}}
\def\$#1\${\begin{align*}#1\end{align*}}
\usepackage{enumitem}
\begin{document}

%\title{\huge Generative Adversarial Imitation Learning in Linear Quadratic Regulator Setting}
%
%\maketitle
%\begin{abstract}
%We provide the convergence analysis of Generative Adversarial Imitation Learning (GAIL) in Linear Quadratic Regulator (LQR) setting. We show that with suitable choice of constant stepsize, the gradient algorithm converges to the saddle point in Q-linear rate. 
%\end{abstract}

\title{\huge On the Global Convergence of Imitation Learning:\\ A Case for Linear Quadratic Regulator}
\author{
Qi Cai
\thanks{Northwestern University}
\qquad
Mingyi Hong
\thanks{University of Minnesota Twin Cities}
\qquad
Yongxin Chen
\thanks{Georgia Institute of Technology}
\qquad
Zhaoran Wang
\footnotemark[1]
}

\maketitle

%!TEX root = main.tex

\begin{abstract}
%\begin{flushleft}
We study the global convergence of generative adversarial imitation learning for linear quadratic regulators, which is posed as minimax optimization. To address the challenges arising from non-convex-concave geometry, we analyze the alternating gradient algorithm and establish its Q-linear rate of convergence to a unique saddle point, which simultaneously recovers the globally optimal policy and reward function. We hope our results may serve as a small step towards understanding and taming the instability in imitation learning as well as in more general non-convex-concave alternating minimax optimization that arises from reinforcement learning and generative adversarial learning. 
%\end{flushleft}

%A case study of imitation learning in linear quadratic control setting is presented. We adopt the Generative Adversarial Imitation Learning (GAIL) framework for imitation learning and prove global convergence of the gradient algorithm for Linear Quadratic Regulator (LQR). We further show that with suitable choice of constant stepsize, the gradient algorithm converges to the saddle point in Q-linear rate. 
\end{abstract}

\section{Introduction}
%\begin{flushleft}
Imitation learning is a paradigm that learns from expert demonstration to perform a task. The most straightforward approach of imitation learning is behavioral cloning \citep{Pom91}, which learns from expert trajectories to predict the expert action at any state. Despite its simplicity, behavioral cloning ignores the accumulation of prediction error over time. Consequently, although the learned policy closely resembles the expert policy at a given point in time, their  trajectories may diverge in the long term. 

To remedy the issue of error accumulation, inverse reinforcement learning \citep{russell1998learning, ng2000algorithms, abbeel2004apprenticeship, ratliff2006maximum, ziebart2008maximum, ho2016generative} jointly learns a reward function and the corresponding optimal policy, such that the expected cumulative reward of the learned policy closely resembles that of the expert policy. In particular, as a unifying framework of inverse reinforcement learning, generative adversarial imitation learning (GAIL) \citep{ho2016generative} casts most existing approaches as iterative methods that alternate between (i) minimizing the discrepancy in expected cumulative reward between the expert policy and the policy of interest and (ii) maximizing such a discrepancy over the reward function of interest. Such a minimax optimization formulation of inverse reinforcement learning mirrors the training of generative adversarial networks (GAN), which alternates between updating the generator and discriminator, respectively. 

Despite its prevalence, inverse reinforcement learning, especially GAIL, is notoriously unstable in practice. More specifically, most inverse reinforcement learning approaches involve (partially) solving a reinforcement learning problem in an inner loop, which is often unstable, especially when the intermediate reward function obtained from the outer loop is ill-behaved. This is particularly the case for GAIL, which, for the sake of computational efficiency, alternates between policy optimization and reward function optimization without fully solving each of them. Moreover, such instability is exacerbated when the policy and reward function are both parameterized by deep neural networks. In this regard, the training of GAIL is generally more unstable than that of GAN, since policy optimization in deep reinforcement learning is often more challenging than training a standalone deep neural network. 

In this paper, we take a first step towards theoretically understanding and algorithmically taming the instability in imitation learning. In particular, under a minimax optimization framework, we for the first time establish the global convergence of GAIL under a fundamental setting known as linear quadratic regulators (LQR). Such a setting of LQR is studied in a line of recent works \citep{bradtke1993reinforcement, fazel2018global, tu2017least, tu2018gap, dean2018regret, dean2018safely, simchowitz2018learning, dean2017sample, hardt2018gradient} as a lens for theoretically understanding more general settings in reinforcement learning. See \cite{recht2018tour} for a thorough review. In imitation learning, particularly GAIL, the setting of LQR captures four critical challenges of more general settings: 
\begin{enumerate}
\item[(i)] the minimax optimization formulation, 
\item[(ii)] the lack of convex-concave geometry, 
\item [(iii)] the alternating update of policy and reward function, and 
\item[(iv)] the instability of the dynamical system induced by the intermediate policy and reward function (which differs from the aforementioned algorithmic instability). 
\end{enumerate}

Under such a fundamental setting, we establish a global sublinear rate of convergence towards a saddle point of the minimax optimization problem, which is guaranteed to be unique and recovers the globally optimal policy and reward function. Moreover, we establish a local linear rate of convergence, which,  combined with the global sublinear rate of convergence, implies a global Q-linear rate of convergence. A byproduct of our theory is the stability of all the dynamical systems induced by the intermediate policies and reward functions along the solution path, which addresses the key challenge in (iv) and plays a vital role in our analysis. At the core of our analysis is a new  potential function tailored towards non-convex-concave minimax optimization with alternating update, which is of independent interest. To ensure the decay of potential function, we rely on the aforementioned stability of intermediate dynamical systems along the solution path. To achieve such stability, we unveil an intriguing ``self-enforcing'' stabilizing mechanism, that is, with a proper configuration of stepsizes, the solution path approaches the critical threshold that separates stable and unstable regimes at a slower rate as it gets closer to such a threshold. In other words, such a threshold forms an implicit barrier, which ensures the stability of the intermediate dynamical systems along the solution path without any explicit regularization. 

Our work extends the recent line of works on reinforcement learning under the setting of LQR \citep{bradtke1993reinforcement, recht2018tour, fazel2018global, tu2017least, tu2018gap, dean2018regret, dean2018safely, simchowitz2018learning, dean2017sample, hardt2018gradient} to imitation learning. In particular, our analysis relies on several geometric lemmas established in \cite{fazel2018global}, which are listed in \S\ref{sec:aux} for completeness. However, unlike policy optimization in reinforcement learning, which involves solving a minimization problem where the objective function itself serves as a good potential function, imitation learning involves solving a minimax optimization problem, which requires incorporating the gradient into the potential function. In particular, the stability argument developed in \cite{fazel2018global}, which is based on the monotonicity of objective functions along the solution path, is no longer applicable, as minimax optimization alternatively decreases and increases the objective function at each iteration. In a broader context, our work takes a first step towards extending the recent line of works on nonconvex optimization, e.g., 
\cite{
baldi1989neural, du2018power, wang2014optimal, wang2014nonconvex, zhao2015nonconvex, ge2015escaping, ge2017no, ge2017learning, anandkumar2014tensor, bandeira2016low, li2016online, li2016symmetry, hajinezhad2016nestt, bhojanapalli2016global, sun2015complete, sun2018geometric}, to non-convex-concave minimax optimization \citep{, du2018linear, sanjabi2018solving, rafi2018noncon, lin2018solving, dai2017learning, dai2018boosting, dai2018sbeed, lu2019understand} with alternating update, which is prevalent in reinforcement learning, imitation learning, and generative adversarial learning, and poses significantly more challenges. 

In the rest of this paper, \S\ref{background} introduces imitation learning, the setting of LQR, and the generative adversarial learning framework. In  \S\ref{algo}, we introduce the minimax optimization formulation and the gradient algorithm. In \S\ref{analysis} and \S\ref{proof}, we present the theoretical results and sketch the proof. We defer the detailed proof to \S\ref{appp1}-\S\ref{sec:aux} of the appendix.

\vskip4pt
\noindent
\textbf{Notation.} 
We denote by $\|\cdot\|$ the spectral norm and $\|\cdot\|_{\tf}$ the Frobenius norm of a matrix. For vectors, we denote by $\|\cdot\|_2$ the Euclidean norm. In this paper, we write parameters in the matrix form, and correspondingly, all the Lipschitz conditions are defined in the Frobenius norm.

\section{Background} \label{background}
In the following, we briefly introduce the setting of LQR in \S\ref{slqr} and imitation learning in \S\ref{silirl}. To unify the notation of LQR and more general reinforcement learning, we stick to the notion of cost function instead of reward function throughout the rest of this paper.

%\subsection{Inverse Reinforcement Learning}
%In Inverse Reinforcement Learning we are given an expert policy that might be the optimal solution to a Reinforcement Learning problem. We are accessible to either the exact policy or a set of trajectories generated by the policy and hope to recover the reward function. 
%%In Imitation Learning problem we only have the trajectories so that we also want to recover the expert policy. 
%%A common target of aforementioned different settings is that we hope the learnt 
%The goal is to obtain a reward function that will guide us to a good policy whose performance is comparable to the expert policy. In model-based setting a single reward function can be obtained by solving an optimization problem with regularization terms. In model-free setting, the algorithms are  required to solve a reinforcement learning problem in the inner loop. 
%A widely used framework for IRL is Max Margin Principle  \citep{ratliff2006maximum}, which maximizes the margin between the reward of expert policy and current solved best policy over reward functions. Some other choices are also possible, including Max Causal Entropy principle \citep{ziebart2008maximum} and Baysian Inverse Reinforcement Learning \citep{vlassis2012bayesian}. Deep learning can also be employed in the inverse reinforcement learning problems \citep{finn2016guided}.

\subsection{Linear Quadratic Regulator}\label{slqr}
%\begin{flushleft}
In reinforcement learning, we consider a Markov decision process $\{\cX,\cU,c,T,\mathbb{D}_0\}$, where an agent interacts with the environment in the following manner. At the $t$-th time step, the agent selects an action $u_t\in\cU$ based on its current state $x_t\in\cX$, and the environment responds with the cost $c_t=c(x_t,u_t)$ and the next state $x_{t+1}\in\cX$, which follows the transition dynamics $T$. Our goal is to find a policy $u_t=\pi_t(x_t)$ that minimizes the expected cumulative cost. In the setting of LQR, we consider $\cX=\real^{d}$ and $\cU=\real^{k}$. The dynamics and cost function take the form
\#
x_{t+1}=Ax_t+Bu_t,\quad c(x_t,u_t)=x_t^\top Q x_t + u_t^\top R u_t,\notag
\#
where $A\in\real^{d\times d}$, $B\in\real^{d\times k}$, $Q\in\real^{d\times d}$, and $R\in\real^{k\times k}$ with $Q, R\succ 0$. The problem of minimizing the expected cumulative cost is then formulated as the optimization problem
\#
\minimize_{\pi_t}~ & \E{ \bigl [{\textstyle\sum_{t=0}^{\infty}} x_t^\top Q x_t + u_t^\top R u_t \bigr ]}  \label{objini}\\
\text{subject to} ~ & x_{t+1}=Ax_t+Bu_t,~u_t=\pi_t(x_t),~x_0 \sim \mathbb{D}_0, \notag
\#
where $\mathbb{D}_0$ is a given initial distribution. Here we consider the infinite-horizon setting with a stochastic initial state $x_0\sim\mathbb{D}_0$. In this setting, the optimal policy $\pi_t$ is known to be static and takes the form of linear feedback $\pi_t(x_t)=-Kx_t$, where $K\in\real^{k\times d}$ does not depend on $t$ \citep{anderson2007optimal}. Throughout the rest of this paper, we also refer to $K$ as policy and drop the subscript $t$ in $\pi_t$. To ensure the expected cumulative cost is finite, we require the spectral radius of $(A-BK)$ to be less than one, which ensures that the dynamical system 
\#\label{eq:trans_dyna_x}
x_{t+1}=Ax_t+Bu_t = (A-BK) x_t
\#
is stable. For a given policy $K$, we denote by $C(K;Q,R)$ the expected cumulative cost in \eqref{objini}. For notational simplicity, we define
\#\label{Sigmadef}
\Sigma_{K}=\E\bigl[{\textstyle\sum_{t=0}^{\infty}} x_t x_t^\top\,|\,\pi_t(x_t)=-Kx_t \bigr],\quad\Sigma_0=\E[x_0 x_0^\top].
\#
By \eqref{Sigmadef}, we have the following equivalent form of $C(K;Q,R)$
\#\label{inner2}
C(K;Q,R)
=\tr(\Sigma_K Q)+\tr(K\Sigma_{K}K^\top R)
=\langle\Sigma_K ,Q\rangle+\langle K\Sigma_{K}K^\top, R\rangle,
\#
where $\langle\cdot,\cdot\rangle$ denotes the matrix inner product. Also, throughout the rest of this paper, we assume that the initial distribution $\mathbb{D}_0$ satisfies $\sigma_{\text{min}}(\Sigma_0)>0$. See \cite{recht2018tour} for a thorough review of reinforcement learning in the setting of LQR.

\subsection{Imitation Learning}\label{silirl}
In imitation learning, we parameterize the cost function of interest by $c(x_t,u_t;\theta)$, where $\theta$ denotes the unknown cost parameter. In the setting of LQR, we have $\theta = (Q,R)$. We observe expert trajectories in the form of $\{(x_t,u_t,c_t)\}^{\infty}_{t=0}$, which are induced by the expert policy $\pi_{\text{E}}$. As a unifying framework of inverse reinforcement learning, GAIL \citep{ho2016generative} casts max-entropy inverse reinforcement learning \citep{ziebart2008maximum} and its extensions as the following minimax optimization problem
\# \label{bgobj}
&\max_{\theta} \min_{\pi} \Bigl\{\E\bigl[{\textstyle\sum_{t=0}^{\infty}}c_t(x_t,u_t;\theta)\,|\,u_t = \pi(x_t)\bigr]-H(\pi)\notag\\&\qquad\qquad\qquad\qquad-\E\bigl[{\textstyle\sum_{t=0}^{\infty}}c_t(x_t,u_t;\theta)\,|\,u_t = \pi_{\text{E}}(x_t)\bigr]-\psi(\theta)\Bigr\},
\#
where for ease of presentation, we restrict to deterministic policies in the form of $u_t = \pi(x_t)$. Here $H(\pi)$ denotes the causal entropy of the dynamical system $\{x_t\}_{t=0}^\infty$ induced by $\pi$, which takes value zero in our setting of LQR, since the transition dynamics in \eqref{eq:trans_dyna_x} is deterministic conditioning on $x_t$. Meanwhile, $\psi(\theta)$ is a regularizer on the cost parameter. 

The minimax optimization formulation in \eqref{bgobj} mirrors the training of GAN \citep{goodfellow2014generative}, which seeks to find a generator of distribution that recovers a target distribution. In the training of GAN, the generator and discriminator are trained simultaneously, in the manner that the discriminator maximizes the discrepancy between the generated and target distributions, while the generator minimizes such a discrepancy. Analogously, in imitation learning, the policy $\pi$ of interest acts as the generator of trajectories, while the expert trajectories act as the target distribution. Meanwhile, the cost parameter $\theta$ of interest acts as the discriminator, which differentiates between the trajectories generated by $\pi$ and $\pi_{\text{E}}$. Intuitively, maximizing over the cost parameter $\theta$ amounts to assigning high costs to the state-action pairs visited more by $\pi$ than $\pi_{\text{E}}$. Minimizing over $\pi$ aims at making such an adversarial assignment of cost impossible, which amounts to making the visitation distributions of $\pi$ and $\pi_{\text{E}}$ indistinguishable.

% In imitation learning, on the one hand, the generative model is the imitation policy which generates a distribution of the state-action pair along its trajectory, and the target distribution is the distributon of the state-action pair in the expert demonstration trajectories. On the other hand, the discriminator is the cost function, which discriminates the two sets of trajectories by giving low costs to the states visited in the expert demonstration while giving high costs to the states visited by the imitation policy. The training stops when the two distributions are close enough such that even the best discriminator can not distinguish the two distributions, that is, there does not exist a feasible cost function that makes the imitation policy get higher costs than the expert demonstration. See \S\ref{formu} for a more detailed discussion about the solution to \eqref{bgobj} specified in LQR. 

%\end{flushleft}

\section{Algorithm} \label{algo}
In the sequel, we first introduce the minimax formulation of generative adversarial imitation learning in \S\ref{formu}, then we present the gradient algorithm in \S\ref{algos}. 
\subsection{Minimax Formulation}\label{formu}
We consider the minimax optimization formulation of the imitation learning problem,
\# \label{obj}
\min_{K\in\cK}\,\max_{\theta\in\Theta}\,
m(K,\theta),\quad\text{where}~~m(K,\theta)=C(K;\theta)-C(K_{\text{E}};\theta)-\psi(\theta).
\#
Here we denote by $\theta=(Q,R)$ the cost parameter, where $Q\in\real^{d\times d}$ and $R\in\real^{k\times k}$ are both positive definite matrices, and $\Theta$ is the feasible set of cost parameters. We assume $\Theta$ is convex and there exist positive constants $\alpha_Q$, $\alpha_R$, $\beta_Q$, and $\beta_R$ such that for any $(Q,R) \in \Theta$, it holds that 
\#\label{eq:walpha}
\alpha_Q I \preceq Q \preceq \beta_Q I,\quad \alpha_R I \preceq R \preceq \beta_R I.
\#
Also, $\cK$ consists of all stabilizing policies, such that $\rho(A-BK)<1$ for all $K\in\cK$, where $\rho$ is the spectral radius defined as the largest complex norm of the eigenvalues of a matrix. The expert policy is defined as $K_{\text{E}}=\argmin_K C(K;\tilde{\theta})$ for an unknown cost parameter $\tilde{\theta}\in\Theta$. However, note that $\tilde{\theta}$ is not necessarily the unique cost parameter such that $K_{\text{E}}$ is optimal. Hence, our goal is to find one of such cost parameters $\theta^*$ that $K_{\text{E}}$ is optimal. The term $\psi(\cdot)$ is the regularizer on the cost parameter, which is set to be $\gamma$-strongly convex and $\nu$-smooth. 

To understand the minimax optimization problem in \eqref{obj}, we first consider the simplest case with $\psi(\theta)\equiv0$. A saddle point $(K^*,\theta^*)$ of the objective function in \eqref{obj}, defined by
\#\label{sddf}
m(K^*,\theta^*)=\min_K m(K,\theta^*)=\max_\theta m(K^*,\theta),
\# 
has the following desired properties.
First, we have that the optimal policy $K^*$ recovers the expert policy $K_{\text{E}}$. By the optimality condition in \eqref{sddf}, we have 
\#\label{sddf2}
C(K^*;\tilde{\theta})-C(K_{\text{E}};\tilde{\theta})=m(K^*;\tilde{\theta})\le m(K^*;\theta^*)= C(K^*;\theta^*)-C(K_{\text{E}};\theta^*)\le0,
\#
where the first inequality follows from the optimality of $\theta^*$ and the second inequality follows from the optimality of $K^*$. Since the optimal solution to the policy optimization problem $\min_{K} C(K;\tilde{\theta})$ is unique (as proved in \S\ref{proof}), we obtain from \eqref{sddf2} that $K^*=K_\text{E}$. Second, $K_\text{E}$ is an optimal policy with respect to the cost parameter $\theta^*$, since by $K^*=K_{\text{E}}$ and the optimality condition $K^*=\argmin_K C(K;\theta^*)$, we have $K_\text{E}=\argmin_K C(K;\theta^*)$. In this sense, the saddle point $(K^*;\theta^*)$ of \eqref{obj} recovers a desired cost parameter and the corresponding optimal policy.

Although $\psi(\cdot)\equiv0$ brings us desired properties of the saddle point, there are several reasons we can not simply discard this regularizer. The first reason is that a strongly convex regularizer improves the geometry of the problem and makes the saddle point of \eqref{obj} unique, which eliminates the ambiguity in learning the desired cost parameter. Second, the regularizer draws connection to the existing optimization formulations of GAN. For example, as shown in \cite{ho2016generative}, with a specific choice of $\psi(\cdot)$, \eqref{obj} reduces to the classical optimization formulation of GAN \citep{goodfellow2014generative}, which minimizes the Jensen-Shannon divergence bewteen the generator and target distributions.

\subsection{Gradient Algorithm} \label{algos}
To solve the minimax optimization problem in \eqref{obj}, we consider the alternating gradient updating scheme,
\#
K_{i+1}&\leftarrow  K_{i}-\eta\cdot \nabla_{K} m(K_i,\theta_i)=K_{i}-\eta\cdot \nabla_{K} C(K_i;\theta_i), \label{pg} \\ 
\theta_{i+1}&\leftarrow\Pi_{\Theta}\Bigl[\theta_i
+\lambda\cdot\bigl(\nabla_Qm(K_{i+1},\theta_{i}),\nabla_Rm(K_{i+1},\theta_{i})\bigr)
\Bigr] \label{cg}.
\#
Here $\Pi_{\Theta}[\cdot]$ is the projection operator onto the convex set $\Theta$, which ensures that each iterate $\theta_i$ stays within $\Theta$. 

%\left(\begin{matrix} 
%	Q_{i+1} \\
%	R_{i+1}
%\end{matrix}\right)&=
%\Pi_{\Theta}\left(
%\left(\begin{matrix} 
%	Q_{i} \\
%	R_{i}
%\end{matrix}\right)+\lambda
%\left(\begin{matrix} 
%	\Sigma_{K_{i+1}}-\Sigma_{K_{\text{E}}} - \nabla_Q \psi(Q_i,R_i) \\
%	K_{i+1}\Sigma_{K_{i+1}}K_{i+1}^\top-K_{\text{E}}\Sigma_{K_{\text{E}}}K_{\text{E}}^\top - \nabla_R \psi(Q_i,R_i)
%\end{matrix}\right)\right)

%\noindent{\bf Deterministic Policy Gradient.} Let $Q^K(x, u(K,x);\theta)$ be the action-value-function corresponding to policy parameters $K$ and cost parameters $\theta$. The gradient $\nabla_K Q^K(x, u(K,x);\theta_i)$ is only taken with respect to $K$ in the argument  $u(K,x)$. Then the gradient descent step is
%\#
%K_{i+1}=K_{i}+\eta \E[ \nabla_K Q^{K_i}(x, u(K_i,x);\theta_i) ],\quad \text{where}~~ u(K,x)=-Kx.
%\#

There are several ways to obatin the gradient in \eqref{pg} without knowing the dynamics $X_{t+1}=Ax_t+Bu_t$ but based on the trajectory $\{(x_t,u_t,c_t)\}_{t=0}^{\infty}$. One example is the deterministic policy gradient algorithm \citep{silver2014deterministic}. In specific, the gradient of the cost function is obtained through the limit
\$
\lim_{\sigma\downarrow0}\E\bigl[\nabla_{K}\pi_{K,\sigma}(u\,|\,x)\cdot Q^{\pi_{K,\sigma}}(x,u)\bigr],
\$
where $\pi_{K,\sigma}(u\,|\,x)$ is a stochastic policy that takes the form $u\,|\,x\sim\cN(-Kx,\sigma^2I)$. Here $Q^{\pi_{K,\sigma}}(x,u)$ is the action-value function associated with the policy $\pi_{K,\sigma}(u\,|\,x)$, defined as the expected total cost of the policy $\pi_{K,\sigma}(u\,|\,x)$ starting at state $x$ and action $u$, which can be estimated based on the trajectory $\{(x_t,u_t,c_t)\}_{t=0}^{\infty}$. An alternative approach is the evolutionary strategy \citep{salimans2017evolution}, which uses zeroth-order information to approximate $\nabla_KC(K;\theta)$ with a random perturbation,
\$
\E_{\text{vec}(\varepsilon)\sim\cN(0,\sigma^2 I)} 
\bigl[
C(K_i+\varepsilon;\theta_i)\cdot\varepsilon
\bigr]\big/\sigma^2, 
\$
where $\varepsilon\in\real^{k\times d}$ is a random matrix in $\real^{k\times d}$ with a sufficiently small variance $\sigma^2$. 

To obtain the gradient in \eqref{cg}, we have
\#
\nabla_Qm(K_{i+1};\theta_i)&=\Sigma_{K_{i+1}}-\Sigma_{K_{\text{E}}} - \nabla_Q \psi(Q_i,R_i),\label{eq:w8}\\
\nabla_Rm(K_{i+1};\theta_i)&=K_{i+1}\Sigma_{K_{i+1}}K_{i+1}^\top-K_{\text{E}}\Sigma_{K_{\text{E}}}K_{\text{E}}^\top - \nabla_R \psi(Q_i,R_i).\label{eq:w9}
\# 
Here $\Sigma_K=\E[\sum_{t=0}^\infty x_tx_t^\top]$ with $\{x_t\}_{t=0}^{\infty}$ generated by policy $K$, which can be estimated based on the trajectory $\{(x_t,u_t,c_t)\}_{t=0}^{\infty}$.

\section{Main Results} \label{analysis}

In this section, we present the convergence analysis of the gradient algorithm in \eqref{pg} and \eqref{cg}. We first prove that the solution path $\{K_i\}_{i\ge0}$ are guaranteed to be stabilizing and then establish the global convergence. For notational simplicity, we define the following constants,
 \# \label{nota1}
 \alpha=\min\{\alpha_Q, \alpha_R\},\quad 
 \sigma_{\theta}=\sup_{(Q,R)\in\Theta}\bigl(\|Q\|^2_{\tf}+\|R\|^2_{\tf}\bigr)^{1/2},\quad
\mu=\sigma_{\text{min}}(\Sigma_0)>0,\#
where $\alpha_Q$ and $\alpha_R$ are defined in \eqref{eq:walpha}, and $\Sigma_0 = \EE[x_0 x_0^\top]$. Also, we define
 \#
 M&=\beta_Q\cdot\tr(\Sigma_{K_0})+\beta_R\cdot\tr(K_0\Sigma_{K_0}{K_0}^\top), \label{nota2}\\
F&=\max\Big\{
\|\Sigma_{K_\text{E}}\|_{\tf}+\sup_{(Q,R)\in\Theta}\|\nabla_Q\psi(Q,R)\|_{\tf}, \|K_{\text{E}}\Sigma_{K_\text{E}}K_{\text{E}}^\top\|_{\tf}+
	\sup_{(Q,R)\in\Theta}\|\nabla_R\psi(Q,R)\|_{\tf} \Bigr \},  \label{nota3}
\#
 which play a key role in upper bounding the cost function $C(K; \theta)$.

\subsection{Stability Guarantee}
A minimum requirement in reinforcement learning is to obtain a stabilizing policy such that the dynamical system does not tend to infinity. Throughout this paper, we employ a notion of uniform stability, which states that there exists a constant $S$ such that $\|\Sigma_{K_i}\|\le S$ for all $i$. Moreover, the uniform stability also allows us to establish the smoothness of $m(K,\theta)$, which is discussed in \S\ref{lipsec}.

Recall that we assume $Q$ and $R$ are positive definite. Therefore, the uniform stability is implied by the boundedness of the cost function $C(K_i;\theta_i)$, since we have
\# \label{ineqs}
\alpha_Q\cdot\|\Sigma_{K_i}\|\le  \alpha_Q\cdot\tr(\Sigma_{K_i})\le \langle \Sigma_{K_i},Q_i \rangle  \le C(K_i;\theta_i),
\#
where the second inequality follows from the properties of trace and the assumption $Q\succeq\alpha_QI$ in \eqref{eq:walpha}. However, it remains difficult to show that the cost function $C(K_i;\theta_i)$ is upper bounded. Although the update of policy in \eqref{pg} decreases the cost function $C(K_i;\theta_i)$, the update of cost parameter increases $m(K_i;\theta_i)$, which possibly increases the cost function $C(K_i;\theta_i)$. To this end, we choose suitable stepsizes as in the next condition to ensure the boundedness of the cost function. 

\begin{cond}\label{con:1} 
	For the update of policy and cost parameter in \eqref{pg} and \eqref{cg}, let
		\$
		\eta \le \min \Biggl\{& \frac{\alpha_Q^{3} \mu^{5/2}(\alpha F+2M)^{-7/2} }
		{16 
			\kappa_1^{1/2} \kappa_2\cdot \|B\| },
		 \frac{\alpha_Q}{32 \kappa_1 (\alpha F+2M)},
		 \frac{2M}{\alpha_Q\alpha_R\mu^2}
		  \Biggr\},\quad
		  \lambda/\eta \le \frac{\alpha_Q\alpha_R\alpha^2\mu^2}{2M(\alpha F+2M)}.
		\$
		Here $\alpha$, $\mu$, $F$, and $M$ are defined in \eqref{nota1}, \eqref{nota2}, and \eqref{nota3}. The constants $\kappa_1$ and $\kappa_2$ are defined as
		\#\label{eq:wa9}
		\kappa_1=\beta_R+(\alpha F+2M) \cdot \|B\|^2/\mu,\quad
		\kappa_2=1+
		(\mu\alpha_Q)^{-1/2} (\alpha F+2M)^{1/2}.
		\#
\end{cond}

The next lemma shows that the solution path $\{K_i\}_{i\ge0}$ is uniformly stabilizing, and meanwhile, along the solution path, the cost function $C(K_i;\theta_i)$ and $\|K_i\|^2$ are both upper bounded. 

\begin{lemma}\label{lm:bc}
	Under Condition \ref{con:1}, we have
	\$%\label{coro1}
	C(K_i;\theta_i)\le \alpha F+2M,\quad\|K_i\|^2\le(\alpha F+2M)/(\alpha_R \mu), \quad\|\Sigma_{K_i}\|\le (\alpha F+2M)/\alpha_Q
	\$
	for all $i\ge0$. 
\end{lemma}
\begin{proof}
	See \S\ref{appp1} for a detailed proof.
\end{proof}

\subsection{Global Convergence}

Before showing the gradient algorithm converges to the saddle point $(K^*,\theta^*)$ of \eqref{obj}, we establish its uniqueness. We define the proximal gradient of the objective function $m(K,\theta)$ in \eqref{obj} as
\# \label{proxg}
L(K,\theta) =
\Bigl(
	\nabla_K m(K,\theta),~
	\theta - \Pi_{\Theta}\bigr[\theta-\nabla_{\theta} m(K,\theta)\bigr]
\Bigr).
\#
Then a proximal stationary point is defined by $L(K,\theta)=0$.
\begin{lemma} [Uniqueness of Saddle Point]\label{usp}
There exists a unique proximal stationary point, denoted as $(K^*,\theta^*)$, of the objective function $m(K,\theta)$ in \eqref{obj}, which is also its unique saddle point.
\end{lemma}
\begin{proof}
See \S\ref{pusp} for a detailed proof.
\end{proof}

 To analyze the convergence of the gradient algorithm, we first need to establish the Lipschitz continuity and smoothness of $m(K,\theta)$. However, the cost function $C(K;\theta)$ becomes steep as the policy $K$ is close to unstabilizing. Therefore, we do not have such desired Lipschitz continuity and smoothness of $\Sigma_K$ and $K\Sigma_KK^\top$ with respect to $K$. However, given $\|\Sigma_K\|$ is upper bounded as in Lemma \ref{lm:bc}, we obtain such desired properties in the following lemma.
 
 For notational simplicity, we slightly abuse the notation and rewrite $\theta$ as a block diagonal matrix and correspondingly define $V(K)$,
 \# \label{innerform}
 \theta=
 \left(\begin{matrix} 
 	Q &0 \\
 	0&R
 \end{matrix}\right)\in\real^{(d+k)\times(d+k)},\quad
 V(K)=
 \left(\begin{matrix} 
 	\Sigma_{K} &0\\
 	0 & K\Sigma_{K}K^\top
 \end{matrix}\right)\in\real^{(d+k)\times(d+k)}.
 \#
 Then the objective function takes the form
 \#\label{svform}
 m(K,\theta)=\bigl\langle V(K),\theta \bigr\rangle-\bigl\langle V(K_{\text{E}}),\theta \bigr\rangle-\psi(\theta).
 \#

\begin{lemma}\label{lm:lip1}
We assume that the initial policy $K_0$ of the gradient algorithm is stabilizing. Under Condition \ref{con:1}, there exists a compact set $\cK^\dagger\subsetneqq\cK$ such that $K_i \in \cK^\dagger$ for all $i\ge0$. Also, there exist constants $\tau_{V}$ and $\nu_V$ such that the matrix-valued function $V(K)$ defined in \eqref{innerform}
is $\tau_{V}$-Lipschitz continuous and $\nu_V$-smooth over $\cK^\dagger$. That is,  for any $K_1,K_2 \in \cK^\dagger$ and $j,\ell\in[d+k]$, we have
%\$
%\|V(K_1)-V(K_2)\|_{\tf}\le \tau_{V}\cdot\|K_1-K_2\|_{\tf},\quad 
%\bigl({\textstyle \sum_{j,\ell=1}^{d+k}}\|\nabla V_{j,\ell}(K_1)-\nabla V_{j,\ell}(K_2)\|_{\tf}^2\bigr)^{1/2} \le \nu_V\cdot\|K_1-K_2\|_{\tf}.
%\$
\$
\|V(K_1)-V(K_2)\|_{\tf}\le \tau_{V}\cdot\|K_1-K_2\|_{\tf},\quad 
 \|\nabla V_{j,\ell}(K_1)-\nabla V_{j,\ell}(K_2)\|_{\tf} \le \nu_V/(d+k) \cdot\|K_1-K_2\|_{\tf}.
\$
\end{lemma}
\begin{proof}
See \S\ref{lipsec} for a detailed proof.
\end{proof}

Note that the cost parameter $(Q,R)$ is only identifiable up to a multiplicative constant. Recall that we assume $\alpha_QI\preceq Q\preceq\beta_QI$ and $\alpha_RI\preceq R\preceq\beta_RI$. In the sequel, we establish the sublinear rate of convergence with a proper choice of $\alpha_Q$, $\alpha_R$, $\beta_Q$, and $\beta_R$, which is characterized by the following condition. 

\begin{cond} \label{cond:cptb}
We assume that $\alpha_Q$, $\alpha_R$, $\beta_Q$, and $\beta_R$ satisfy
\# \label{cond:cptbineq}
\alpha_Q\alpha_R\alpha^2\gamma\geq 14\sigma_{\theta}\nu_V M(\alpha F+2M), 
\#
where $F$, $M$, and $\sigma_\theta$ are defined in \eqref{nota1}, \eqref{nota2} and \eqref{nota3}, and $\nu_V$ is defined in Lemma \ref{lm:lip1}.
\end{cond}

The following condition, together with Condition \ref{con:1}, specifies the required stepsizes to establish the global convergence of the gradient algorithm.

\begin{cond}\label{con100}
For the stepsizes $\eta$ and $\lambda$ in \eqref{pg} and \eqref{cg}, let
\$
 \eta\le\min\biggl\{\frac{1}{100\tau_{V}},\frac{1}{2\sigma_\theta \nu_V}\biggr\},\quad
 \lambda\le\min\biggl\{ \frac{1}{100(\tau_{V}+\nu)},\frac{3\nu_V\sigma_\theta}{100\tau_{V}^2},\frac{\gamma}{100\nu^2} \biggr \},\quad
 \eta/\lambda < \frac{\gamma}{7\nu_V\sigma_{\theta}} .
\$
\end{cond}

In the following, we establish the global convergence of the gradient algorithm. Recall that as defined in \eqref{proxg}, $L(K,\theta)$ is the proximal gradient of the objective function $m(K,\theta)$ defined in \eqref{obj}.

\begin{theorem}\label{mainthm}
Under Conditions \ref{con:1}, \ref{cond:cptb}, and \ref{con100}, we have $\lim_{i\rightarrow\infty}\|L(K_i,\theta_i)\|_{\tf}=0$,
which implies that $\{(K_i,\theta_i)\}_{i=0}^\infty$ converges to the unique saddle point $(K^*,\theta^*)$ of $m(K,\theta)$. To characterize the rate of convergence, we define $\Gamma(\varepsilon)$ as the smallest iteration index that $\|L(K_i,\theta_i)\|_{\tf}^2$ is below an error $\varepsilon>0$,
\#\label{Gammadefi} 
\Gamma(\varepsilon)= \min \bigl\{i\,|\,\| L(K_i,\theta_i)\|_{\tf}^2\le\varepsilon \bigr\}.
\#
Then there exists a constant $\zeta$, which depends on $K_0$, $\theta_0$, $\eta$, and $\lambda$ (as specified in \eqref{zetadef}), such that $\Gamma(\varepsilon)\le \zeta/\varepsilon$ for any $\varepsilon$.
\end{theorem}
\begin{proof}
See \S\ref{lipsec} for a detailed proof.
\end{proof}

To understand Condition \ref{cond:cptb}, we consider a simple case where the regularizer $\psi(\cdot)$ is the squared penalty centered at some point $(\bar{Q},\bar{R})\in\Theta$, that is,
\$
\psi(Q,R)=\gamma \cdot \bigl(\|Q-\bar{Q}\|_{\tf}^2+\|R-\bar{R}\|_{\tf}^2\bigr).
\$
Then we have $\|\nabla\psi(Q,R)\|_{\tf}\le2\gamma\omega$, where $\omega=\sup_{\theta,\theta'\in\Theta}\|\theta-\theta'\|_{\tf}$. Also, by \eqref{nota3} we have
\# \label{scse3}
F\le\max\bigl\{
\|\Sigma_{K_\text{E}}\|_{\tf}, \|K_{\text{E}}\Sigma_{K_\text{E}}K_{\text{E}}^\top\|_{\tf}\bigr\}+2\gamma\omega.
\#
Let $\max\{\beta_Q,\beta_R\}/\alpha\le \iota$ for some constant $\iota$. By \eqref{nota2} we have
\# \label{scse2}
M\le\iota\alpha\cdot\bigl(\tr(\Sigma_{K_0})+\tr(K_0\Sigma_{K_0}{K_0}^\top)\bigr).
\#
By \eqref{scse2} and Lemma \ref{lm:bc}, we obatin
\# 
\|\Sigma_{K_i}\|&\le (\alpha F+2M)/\alpha_Q\le F+2\iota\cdot\bigl(\tr(\Sigma_{K_0})+\tr(K_0\Sigma_{K_0}{K_0}^\top)\bigr),\label{scse}\\
\|K_i\|^2&\le(\alpha F+2M)/(\alpha_R \mu)\le \Bigl(F+2\iota\cdot\bigl(\tr(\Sigma_{K_0})+\tr(K_0\Sigma_{K_0}{K_0}^\top)\bigr)\Bigr)\Big/\mu \label{scse6}
\#
for all $i\ge0$. In \S\ref{lipsec} we further prove that $\nu_V$ is determined by the uniform upper bound of $\|\Sigma_{K_i}\|$ and $\|K_i\|$ along the solution path, which is established in Lemma \ref{lm:bc}. Hence, by \eqref{scse} and \eqref{scse6} we have that $\nu_V$ is independent of $\alpha$. Meanwhile, by \eqref{scse3} and \eqref{scse2} we have
\$
14\sigma_{\theta} M(\alpha F+2M)=O(\alpha^3).
\$
Thus, for a sufficiently large $\alpha$, we have
\$
\alpha_Q\alpha_R\alpha^2\gamma \ge \alpha^4\gamma\ge 14\sigma_{\theta}\nu_V M(\alpha F+2M),
\$
which leads to Condition \ref{cond:cptb}.

Condition \ref{cond:cptb} plays a key role in establishing the convergence. On the one hand, to ensure the boundedness of the cost function $C(K_i;\theta_i)$, we require an upper bound of $\lambda/\eta$ in Condition \ref{con:1}. On the other hand, to ensure the convergence of the gradient algorithm, we require an upper bound of $\eta/\lambda$ in Condtion \ref{con100}. Condition \ref{cond:cptb} ensures such two requirements on stepsizes are compatible.

\subsection{Q-Linear Convergence}
In this section, we establish the Q-linear convergence of the gradient algorithm in \eqref{pg} and \eqref{cg}. Recall that the optimal policy takes the form $K^*=(B^\top PB+R)^{-1}B^\top PA$, where $P$ is the positive definite solution to the discrete-time algebraic Riccati equation \citep{anderson2007optimal},
\# \label{ip}
f(P,Q,R)=P-A^{\top}PA-Q+A^\top PB(B^\top PB+R)^{-1} B^\top PA=0.
\#
We denote by $P^*(Q,R)$ the corresponding implicit matrix-valued function defined by \eqref{ip}. Also, we define $Y\in\real^{d^2\times d^2}$ as
\$
Y_{(i-1) d+j,(k-1) d+\ell}=\frac{\partial f_{i,j}}{\partial P_{k,\ell}}\bigl(P^*(Q^*,R^*),Q^*,R^*\bigr),
\$
for $i,j,k,\ell\in[d]$. We assume the following regularity condition on $f$.

\begin{cond} \label{con:2} 
The unique stationary point of cost parameter $(Q^*,R^*)$ is an interior point of $\Theta$.
Also, we assume that $\det(Y)\neq 0$.
\end{cond}

We define $K^*(\theta)$ as the unique optimal policy corresponding to the cost parameter $\theta$ and denote by $m^*(\theta)$ the corresponding value of the objective function $m(K,\theta)$ defined in \eqref{obj}, that is,
\#\label{kstar}
K^*(\theta)=\argmin_{K\in\cK}C(K;\theta),\quad m^*(\theta)=m\bigl(K^*(\theta), \theta\bigr).
\#
The following two lemmas characterize the local properties of the functions $K^*(\theta)$ and $m^*(\theta)$ in a neighborhood of the saddle point $(K^*,\theta^*)$ of $m(K,\theta)$. 

\begin{lemma}\label{qclm}
	Under Condition \ref{con:2}, there exist constants $\tau_{K^*}$ and $\nu_{K^*}$, and a neighborhood $\cB_K$ of $\theta^*$, such that $K^*(\theta)$ is $\tau_{K^*}$-Lipschitz continuous and $\nu_{K^*}$-smooth with respect to $\theta\in\cB_K$.
\end{lemma}
\begin{proof}
See \S\ref{pqclm} for a detailed proof.
\end{proof}

\begin{lemma}\label{lm:lip3}
	Under Condition \ref{con:2}, there exist a constant $\nu_{m^*}$ and a neighborhood $\cB_{m^*}$ of $\theta^*$ such that $m^*(\theta)$ is $\gamma$-strongly concave and $\nu_{m^*}$-smooth with respect to $\theta\in\cB_{m^*}$.
\end{lemma}
\begin{proof}
See \S\ref{plm:lip3} for a detailed proof.
\end{proof}

To establish the Q-linear convergence, we need an additional condition, which upper bounds the stepsizes $\eta$ and $\lambda$.
\begin{cond} \label{con:w686} 
For the stepsizes $\eta$ and $\lambda$ in \eqref{pg} and \eqref{cg}, let
\$
% \eta\le {1}/{\nu_{K^*}},\quad \lambda\le\min\bigl\{ 1/\gamma, 2/(\nu_{m^*} - \gamma) \bigr \}.
 %\eta\le \min\biggl\{\frac{1}{\alpha_R \mu}, \frac{2}{\nu_{K^*} - \alpha_R \mu}\biggr\},\quad \lambda\le\min\biggl\{ \frac{1}{\gamma}, \frac{2}{\nu_{m^*} - \gamma} \biggr\}.
 \eta\le 2/(\alpha_R\mu+\nu_{K^*}),\quad
 \lambda\le 2/(\gamma+\nu_{m^*}).
\$
\end{cond}
We define the following potential function
\#\label{eq:wz1}
Z_{i}=\| \theta_{i}-\theta^* \|_{\tf}+a\cdot\|K_{i}-K^*(\theta_{i})\|_{\tf},
\# where $a=\gamma/(3\tau_{K^*} \nu_{m^*})$. Note that $\lim_{i\rightarrow\infty}Z_i=0$ implies that $\{(K_i,\theta_i)\}_{i=0}^\infty$ converges to $(K^*,\theta^*)$, since we have $K^*(\theta^*)=K^*$. Also, we define
\# \label{qdef}
\upsilon=\max\bigl\{
		1-\lambda \gamma + a\cdot \lambda \nu_{m^*} \tau_{K^*}, (1- \eta\alpha_R \mu) \cdot ( 1+ \lambda \tau_{V}/a +
		\lambda \tau_{V} \tau_{K^*} )  
		\bigr\}.
\#
The following theorem establishes the Q-linear convergence of the gradient algorithm.

\begin{theorem}{\label{thm:lr}}
	Under Conditions \ref{con:1}, \ref{cond:cptb}, \ref{con100}, \ref{con:2}, and \ref{con:w686}, we have $\upsilon\in(0,1)$ in \eqref{qdef}. There exists an iteration index $N>0$ such that $Z_{i+1}\le \upsilon\cdot Z_i$ for all $i>N$. 
\end{theorem}
\begin{proof}
See \S\ref{qlsketch} for a detailed proof.
\end{proof}

\section{Proof Sketch} \label{proof}
In this section, we sketch the proof of the main results in \S\ref{analysis}.

 \subsection{Proof of Stability Guarantee}
 
To prove Lemma \ref{lm:bc}, we lay out two auxiliary lemmas that characterize the geometry of the cost function $C(K;\theta)$ with respect to $K$. The first lemma characterizes the stationary point of policy optimization. The second lemma shows that $C(K;\theta)$ is gradient dominated with respect to $K$.
 
 \begin{lemma} \label{lm:stp}
 	If $\nabla_K C(K;\theta)=0$, then $K$ is the unique optimal policy corresponding to the cost parameter $\theta$.
 \end{lemma}
 \begin{proof}
 See \S\ref{plm:stp} for a detailed proof.
 \end{proof}
 
\begin{lemma}[Corollary 5 in \cite{fazel2018global}]{\label{lm:gd}}
	The cost function $C(K;\theta)$ is gradient dominated with respect to $K$, that is, 
	\$
	C(K;\theta) - C(K^*(\theta);\theta) \leq \mu_C \cdot \|\nabla_K C(K;\theta)\|^2_{\tf},
	\$
	where $\mu_C=\|\Sigma_{K^*(\theta)}\|/(\mu^2\sigma_{\min}(R))$ and $K^*(\theta)$ is defined in \eqref{kstar}.
\end{lemma}
\begin{proof}
	See \cite{fazel2018global} for a detailed proof.
\end{proof}

Lemma \ref{lm:gd} allows us to upper bound the increment of the cost function at each iteration in \eqref{pg} and \eqref{cg} by choosing a sufficiently small $\lambda$ relative to $\eta$. In fact, we construct a threshold such that when $C(K_{i};\theta_i)$ is close to such a threshold, an upper bound of the increment $C(K_{i+1};\theta_{i+1})-C(K_{i};\theta_i)$ goes to zero. Thus, $C(K_{i};\theta_i)$ is upper bounded by such a threshold. See \S\ref{appp1} for a detailed proof.

\subsection{Proof of Global Convergence} \label{lipsec}

To prove Theorem \ref{mainthm}, we first establish the Lipschitz continuity and smoothness of $m(K,\theta)$ in $K$ within a restricted domain $K^\dagger$ as in Lemma \ref{lm:lip1}. Recall that $V(K)$ is defined in \eqref{innerform} and $m(K,\theta)$ takes the form in \eqref{svform}. Since the matrix-valued function $\Sigma_{K}$ plays a key role in $V(K)$, in the sequel we characterize the smoothness of $\Sigma_{K}$ with respect to $K$ within a restricted set. 
\begin{lemma}\label{lm:lip}
	For any constant $S>0$, there exist constants $\tau_{\Sigma}$ and $\nu_{\Sigma}$ depending on $S$ such that 
		\$\|\Sigma_{K}-\Sigma_{K'}\|_{\tf} \le \tau_{\Sigma} \cdot \|K-K'\|_{\tf},\quad
		\|\nabla_K(\Sigma_{K})_{j, \ell}-\nabla_K(\Sigma_{K'})_{j, \ell}\|_{\tf} \le \nu_{\Sigma}/d \cdot \|K-K'\|_{\tf}\$
		for any $K,K' \in \{K\in\real^{k\times d} \,|\,\|\Sigma_K\|\le S\}$ and $j, \ell \in [d]$. 
\end{lemma}
\begin{proof}
See \S\ref{plm:lip} for a detailed proof.
\end{proof}

 Based on Lemmas \ref{lm:lip} and \ref{lm:bc}, we now prove Lemma \ref{lm:lip1}. 
 \begin{proof}
Let the set $K^{\dagger}$ in Lemma \ref{lm:lip1} be
\$
\cK^\dagger=\bigl\{K\in\real^{k\times d} \,|\, \|\Sigma_K\|\le (\alpha F+2M)/\alpha_Q \bigr\}.
\$
 Then by Lemma \ref{lm:bc}, we have 
 \$
 \|\Sigma_{K_i}\| \le C(K_i;\theta_i)/\alpha_Q \le \alpha (F+2M)/\alpha_Q
 \$
  for all $i\ge0$, which implies $K_i\in\cK^\dagger$ for all $i\ge0$. By Lemma \ref{lm:lip}, we obtain the Lipschitz continuity and smoothness of $\Sigma_{K}$ over $\cK^\dagger$. Furthermore, by the definition of $V(K)$ in \eqref{innerform} and the boundedness of $K_i$ established in Lemma \ref{lm:bc}, $V(K)$ is also Lipschitz continuous and smooth over $\cK^\dagger$. Thus, we conclude the proof of Lemma \ref{lm:lip1}.
\end{proof}

Based on Lemma \ref{lm:lip1}, we prove the global convergence in Theorem \ref{mainthm}. To this end, we construct a potential function that decays monotonically along the solution path, which takes the form
\#\label{potdef}
P_i=
m(K_i,\theta_i)+s\cdot \bigl((1+\eta \nu_V \sigma_{\theta})/2\cdot\|K_{i+1}-K_{i}\|_{\tf}^2
+(\eta/\lambda-\eta\gamma + \eta\lambda \nu^2 )/2\cdot
\|\theta_{i+1}-\theta_i\|_{\tf}^2\bigr)
\#
for some constant $s>0$, which is specified in the next lemma. Meanwhile, we define three constants $\phi_1$, $\phi_2$, and $\phi_3$ as
\#
\phi_1&=1/(2\eta)-\tau_{V}/2-s\cdot (\eta\lambda \tau_{V}^2+3\eta \nu_V \sigma_{\theta}),\label{eq:w1}\\
\phi_2&=s\cdot (\eta\gamma-\eta\lambda \nu^2)/2-
		(1/\lambda+\tau_{V}+\nu)/2,\label{eq:w2}\\
\phi_3&=s \cdot (\eta\gamma-\eta\lambda \nu^2)/2-
		(1/\lambda+\nu)/2 \label{eq:w3}.
\#
The following lemma characterizes the decrement of the potential function defined in \eqref{potdef} at each iteration. 

\begin{lemma} \label{lm:dpf}
Under Conditions \ref{con:1}, \ref{cond:cptb}, and \ref{con100},  we have
\# \label{potd}
P_{i+1}-P_i \le -\phi_1\cdot\|K_{i+1}-K_i\|_{\tf}^2 -\phi_2 \cdot\|\theta_{i+1}-\theta_i\|_{\tf}^2 -\phi_3\cdot\|\theta_i-\theta_{i-1}\|_{\tf}^2.
\#
 Moreover, we have $\phi_1,\phi_2,\phi_3>0$ for $s=12/(13\eta^2\nu_V\sigma_\theta)$.
\end{lemma}
\begin{proof}
See \S\ref{appp2} for a detailed proof.
\end{proof}

Based on Lemma \ref{lm:dpf}, we now prove Theorem \ref{mainthm}.

\begin{proof}
	By the definitions of $P_i$ and $m(K,\theta)$ in \eqref{potdef} and \eqref{obj}, we have $P_i\ge\underline{P}$ for all $i\ge0$, where $\underline{P}$ is 
	\$
	\underline{P}=\inf_{\theta\in\Theta}\bigl\{ -C(K_{\text{E}};\theta)-\psi(\theta)\bigr\} > -\infty.
	\$
	Here we use the fact $C(K;\theta)\ge0$ for any $K\in\cK$ and $\theta\in\Theta$. Let $\phi=1/\min\{\phi_1,\phi_2\}$, where $\phi_1$ and $\phi_2$ are defined in \eqref{eq:w1} and \eqref{eq:w2}. By rearranging the terms in \eqref{potd}, we obtain
	\$
	\sum_{i=0}^{\infty} \|K_{i+1}-K_i\|_{\tf}^2+\|\theta_{i+1}-\theta_i\|_{\tf}^2 \le \sum_{i=0}^{\infty} \phi \cdot (P_{i}-P_{i+1})\le \phi\cdot (P_0-\underline{P})< \infty.
	\$
	By \eqref{pg} and \eqref{cg}, we have
	\$
	\sum_{i=0}^{\infty} \|  L(K_i,\theta_i)  \|_{\tf}^2 \le
	\sum_{i=0}^{\infty} \phi'\cdot \bigl( \|K_{i+1}-K_i\|_{\tf}^2+\|\theta_{i+1}-\theta_i\|_{\tf}^2\bigr) \le  \phi' \phi\cdot (P_0-\underline{P})< \infty, 
	\$
	where $\phi'=\max\{1, 1/\eta^2, 1/\lambda^2 \}$, which implies that $\{\|  L(K_i,\theta_i)  \|_{\tf}\}_{i=0}^\infty$ converges to zero. Also, let
	\#\label{zetadef}
	\zeta=\phi' \phi \cdot (P_0-\underline{P}).
	\# 
	For any $\varepsilon>0$, by the definition of $\Gamma(\varepsilon)$ in \eqref{Gammadefi}, we have
	\$
	\varepsilon\cdot \Gamma(\varepsilon)\le\sum_{i=0}^{\Gamma(\varepsilon)}\|  L(K_i,\theta_i)  \|_{\tf}^2 \le \phi' \phi \cdot (P_0-\underline{P})=\zeta,
	\$
	which implies $\Gamma(\varepsilon)\le \zeta/\varepsilon$. Hence, we conclude the proof of Lemma \ref{mainthm}.
\end{proof}

\subsection{Proof of Q-Linear Convergence}\label{qlsketch} 	
Theorem \ref{mainthm} states that $\{(K_i,\theta_i)\}_{i=0}^\infty$ converges to the unique saddle point $(K^*,\theta^*)$ starting from any stabilizing initial policy$(K_0,\theta_0)$. To establish the Q-linear rate of convergence in Theorem \ref{thm:lr}, we first prove that the cost function $C(K;\theta)$ is locally strongly convex over a neighborhood of the optimal policy $K^*(\theta)$ in the following lemma. Recall that $K^*(\theta)$ is defined in \eqref{kstar}.

\begin{lemma}\label{lm:ksc}
For any cost parameter $\theta\in\Theta$, its corresponding optimal policy $K^*(\theta)$ has a neighborhood $\cK^*_{\theta}$ such that $C(K;\theta)$ is $(\alpha_R \mu)$-strongly convex with respect to $K\in\cK^*_{\theta}$.
\end{lemma}
\begin{proof}
See \S\ref{plm:ksc} for a detailed proof.
\end{proof}

%The following lemma presents the linear convergence of the gradient algorithm on a strongly convex and smooth objective function.

With Lemmas \ref{qclm}, \ref{lm:lip3}, and \ref{lm:ksc}, we establish Theorem \ref{thm:lr} based on the local strongly convex-concave property of $m(K,\theta)$ defined in \eqref{obj}. See \S\ref{aplr} for a detailed proof.

\bibliographystyle{ims}
\bibliography{graphbib}

\newpage 
\begin{appendix}
\section{Proof of Lemma \ref{lm:bc}}\label{appp1}
To prove Lemma \ref{lm:bc}, we lay out the following auxiliary lemmas. Recall that $\alpha$, $\mu$, $F$, and $M$ are defined in \eqref{nota1}, \eqref{nota2}, and \eqref{nota3}. The first lemma establishes an upper bound of the cost function evaluated at the optimal policy.

\begin{lemma} \label{bdlm1}
For any $(Q,R)\in\Theta$ and $K^*(Q,R)$ being the optimal policy corresponding to the cost parameter $(Q, R)$, it holds that
\$
C\bigl(K^*(Q,R);Q,R\bigr) \le M,\qquad
\|\Sigma_{K^*(Q,R)}\| \le M/\alpha_Q.
\$
\end{lemma}

\begin{proof}
By the optimality of $K^*(Q,R)$, we have
\# \label{333}
		C\bigl(K^*(Q,R);Q,R\bigr) &\le C(K_0;Q,R)\nonumber 
		=\langle \Sigma_{K_0},Q\rangle+\langle K_0\Sigma_{K_0}K_0^\top,R\rangle\nonumber \\
		&\le \beta_Q\cdot\tr(\Sigma_{K_0})+\beta_R\cdot\tr(K_0\Sigma_{K_0}{K_0}^\top)=M<\infty.
\#
To obtain the second inequality, we use the assumption that $Q \succeq \alpha_Q I$ in \eqref{eq:walpha}, which implies
\$
 \|\Sigma_{K^*(Q,R)}\| \le C\bigl(K^*(Q,R);Q,R\bigr)\big/\sigma_{\text{min}}(Q) \le M/\alpha_Q.
\$
Thus, we conclude the proof of Lemma \ref{bdlm1}.
\end{proof}

The following lemma characterizes the decrement of the cost function at each iteration of policy optimization. Let $P_{K}$ be the cost-to-go matrix such that 
	\#\label{eq:wa6}
	\E\bigl[ {\textstyle\sum_{t=0}^{\infty}} (x_t^\top Q x_t + u_t^\top R u_t) \,|\,x_0=x, u_t=-Kx_t \bigr]=x^\top P_{K}x
	\# 
	for all $x\in\real^d$.

\begin{lemma}[Lemma 21 in \cite{fazel2018global}]\label{bdlm2} 
	For the update of policy in \eqref{pg}, we have 
	\$
		&C(K_{i+1};Q_{i},R_{i})-C(K_{i};Q_{i},R_{i})\\
		&\quad \le
		-\eta\cdot \sigma_{\text{min}}(R_i)\cdot\mu^2/\|\Sigma_{K^*_i}\| \cdot \Bigl(C(K_{i};Q_{i},R_{i})-C\bigl(K^*(Q_i, R_i); Q_{i},R_{i}\bigr)\Bigr), 
	\$
	where the stepsize $\eta$ satisfies
	\# \label{etacond1}
	\eta &\leq 1/16 \cdot  \min\biggl\{
	\left(\frac{\mu \cdot \sigma_{\min}(Q_i) }{C(K_i;Q_i,R_i)} \right)^2 \cdot 
	\frac{1}{ \|B\| \cdot \| \nabla_K C(K_i;Q_i,R_i)\| \cdot \bigl(1+\|A-B K_i\|\bigr) }, \nonumber\\
	&\qquad\qquad\qquad\qquad \frac{\sigma_{\min}(Q_i)}{2 C(K_i;Q_i,R_i) \cdot \|R + B^\top P_{K_i}B\|}
	\biggr\}.
	\#
\end{lemma}
\begin{proof}
See the appendix of \cite{fazel2018global} for a detailed proof.
\end{proof}

The following lemma upper bounds the increment of the cost function at each iteration of cost parameter optimization. 

\begin{lemma}\label{bdlm3}
For the update of cost parameter in \eqref{cg}, we have
\$
C(K_{i+1};Q_{i+1},R_{i+1})-C(K_{i+1};Q_{i},R_{i})
\le \lambda/\alpha^2\cdot C(K_{i};Q_{i},R_{i})^2+\lambda F/\alpha \cdot C(K_{i};Q_{i},R_{i}).
\$
\end{lemma}
\begin{proof}
By \eqref{cg}, \eqref{eq:w8}, and \eqref{eq:w9}, we have  
	\#
		\|Q_{i+1}-Q_{i}\|_{\tf}&\le
		 \lambda\cdot \bigl(\|\Sigma_{K_{i+1}}\|_{\tf}+ \|\Sigma_{K_{\text{E}}}\|_{\tf}+\|\nabla_Q\psi(Q_i,R_i)\|_{\tf}\bigr), \label{jugg1}\\
		\|R_{i+1}-R_{i}\|_{\tf}&\le
		 \lambda\cdot \bigl(\|K_{i+1}\Sigma_{K_{i+1}}K_{i+1}^\top\|_{\tf}+\|K_{\text{E}}\Sigma_{K_{\text{E}}}K_{\text{E}}^\top\|_{\tf}+\|\nabla_R\psi(Q_i,R_i)\|_{\tf}\bigr).\label{jugg2}
	\#
	Meanwhile, by \eqref{inner2} the increment of the cost function takes the form
	\$
	&C(K_{i+1};Q_{i+1},R_{i+1})-C(K_{i+1};Q_{i},R_{i})\\
	&
	\quad =\langle \Sigma_{K_{i+1}}, Q_{i+1}-Q_{i} \rangle +
		\langle K_{i+1}\Sigma_{K_{i+1}}K_{i+1}^\top, R_{i+1}-R_{i} \rangle.
	\$
	By \eqref{jugg1}, \eqref{jugg2}, and the Cauchy-Schwarz inequality, we obtain
	\# \label{cai910}
	&C(K_{i+1};Q_{i+1},R_{i+1})-C(K_{i+1};Q_{i},R_{i})\nonumber\\
	&\quad\le \lambda\cdot \|\Sigma_{K_{i+1}}\|_{\tf}\cdot\bigl(\|\Sigma_{K_{i+1}}\|_{\tf}+ \|\Sigma_{K_{\text{E}}}\|_{\tf}+\|\nabla_Q\psi(Q_i,R_i)\|_{\tf}\bigr) \notag \\
	&\quad\qquad + \lambda\cdot \|K_{i+1}\Sigma_{K_{i+1}}K_{i+1}^\top\|_{\tf} \cdot \bigl(\|K_{i+1}\Sigma_{K_{i+1}}K_{i+1}^\top\|_{\tf}+\|K_{\text{E}}\Sigma_{K_{\text{E}}}K_{\text{E}}^\top\|_{\tf}+\|\nabla_R\psi(Q_i,R_i)\|_{\tf}\bigr) \notag\\
	&\quad\le \lambda\cdot \bigl(\|\Sigma_{K_{i+1}} \|_\tf+ \| K_{i+1}\Sigma_{K_{i+1}}K_{i+1}^\top \|_\tf\bigr)^2
	+ \lambda\cdot \|\Sigma_{K_{i+1}}\|_{\tf}\cdot\bigl( \|\Sigma_{K_{\text{E}}}\|_{\tf}+\|\nabla_Q\psi(Q_i,R_i)\|_{\tf}\bigr) \notag \\
	&\quad\qquad + \lambda\cdot \|K_{i+1}\Sigma_{K_{i+1}}K_{i+1}^\top\|_{\tf} \cdot \bigl(\|K_{\text{E}}\Sigma_{K_{\text{E}}}K_{\text{E}}^\top\|_{\tf}+\|\nabla_R\psi(Q_i,R_i)\|_{\tf}\bigr).
	\#
	We rearrange the terms in \eqref{cai910} using the definition of $F$ in \eqref{nota3} and obtain
	\# \label{ube1}
		&C(K_{i+1};Q_{i+1},R_{i+1})-C(K_{i+1};Q_{i},R_{i})\nonumber\\
		&\quad\le \lambda\cdot \bigl(\|\Sigma_{K_{i+1}} \|_\tf+ \| K_{i+1}\Sigma_{K_{i+1}}K_{i+1}^\top \|_\tf\bigr)^2+\lambda F\cdot\bigl(\|\Sigma_{K_{i+1}} \|_\tf+ \| K_{i+1}\Sigma_{K_{i+1}}K_{i+1}^\top \|_\tf\bigr).
	\#
	Following from Lemma \ref{bdlm2}, we obtain
	\#\label{eq:wa88}
	C(K_{i};Q_{i},R_{i})&>C(K_{i+1};Q_{i},R_{i}) \geq \alpha_Q \cdot \|\Sigma_{K_{i+1}} \|_{\tf} + \alpha_R\cdot \| K_{i+1}\Sigma_{K_{i+1}}K_{i+1}^\top \|_{\tf} \notag\\
	& \ge  \alpha\cdot \bigl(\|\Sigma_{K_{i+1}} \|_{\tf}+ \| K_{i+1}\Sigma_{K_{i+1}}K_{i+1}^\top \|_{\tf}\bigr),
	\#
	where the third inequality follows from the definition of $\alpha$ in \eqref{nota1}. Plugging \eqref{eq:wa88} into \eqref{ube1}, we obtain
	\$
	C(K_{i+1};Q_{i+1},R_{i+1})-C(K_{i+1};Q_{i},R_{i})
	\le \lambda/\alpha^2\cdot C(K_{i};Q_{i},R_{i})^2+\lambda F/\alpha\cdot C(K_{i};Q_{i},R_{i}),
	\$
	which concludes the proof of Lemma \ref{bdlm3}.
\end{proof}

Based on Lemmas \ref{bdlm1}-\ref{bdlm3}, we now prove Lemma \ref{lm:bc}.
\begin{proof}
	By Lemmas \ref{bdlm1} and \ref{bdlm2}, we have
	\#\label{eq:w88}
		C(K_{i+1};Q_{i},R_{i})-C(K_{i};Q_{i},R_{i})&\le
		-\eta \mu^2 \cdot \sigma_{\text{min}}(R_i)/\|\Sigma_{K^*_i}\|\cdot \Bigl(C(K_{i};Q_{i},R_{i})-C\bigl(K^*(Q_i, R_i);Q_{i},R_{i}\bigr)\Bigr)\notag\\
		&\le -\eta\cdot \alpha_Q\alpha_R\mu^2/M\cdot \Bigl(C(K_{i};Q_{i},R_{i})-C\bigl(K^*(Q_i, R_i);Q_{i},R_{i}\bigr)\Bigr)\notag\\
		&\le -\eta\cdot \alpha_Q\alpha_R\mu^2/M \cdot C(K_{i};Q_{i},R_{i})+\eta\cdot \alpha_Q\alpha_R\mu^2,
	\#
	when $\eta$ is as specified in \eqref{etacond1} of Lemma \ref{bdlm2}.
	
	Meanwhile, recall that in Condition \ref{con:1} we assume
	\# \label{superh1}
	\lambda/\eta \le \frac{\alpha^2\alpha_Q\alpha_R\mu^2}{2M(\alpha F+2M)}.
	\#
	We prove Lemma \ref{lm:bc} by induction. First, we assume that for all $i\le j$, the following two inequalities hold,
	\# \label{superh2}
	\lambda/\alpha^2\cdot C(K_i;Q_i,R_i) \le \eta \cdot \alpha_Q\alpha_R\mu^2/(2M),\quad 
	C(K_i;Q_i,R_i)\le \alpha F+2M.
	\#
    	For $i=0$, \eqref{superh2} follows from the fact that $C(K_0;Q_0,R_0)\le M$ in \eqref{333}. For $i=j+1$, by combining Lemma \ref{bdlm3} and the first inequality in the induction assumption in \eqref{superh2}, we have
	\# \label{337}
	C(K_{j+1};Q_{j+1},R_{j+1})-C(K_{j+1};Q_{j},R_{j})
	\le \eta \cdot \alpha_Q\alpha_R\mu^2/(2M) \cdot \bigl(C(K_{j};Q_{j},R_{j})+\alpha F\bigr).
	\#
	Following from \eqref{eq:w88} and \eqref{337}, the increment of the cost function has the following upper bound, 
	\#\label{341}
	&C(K_{j+1};Q_{j+1},R_{j+1})-C(K_{j};Q_{j},R_{j})\notag\\
	&
	\quad \le -\eta \cdot \alpha_Q\alpha_R\mu^2/(2M) \cdot \bigl(C(K_{j};Q_{j},R_{j})-(\alpha F+ 2M)\bigr),
	\#
	which implies 
	\$
	C(K_{j+1};Q_{j+1},R_{j+1})\le\alpha F+ 2M
	\$
	 for $\eta\le 2M/(\alpha_Q\alpha_R\mu^2)$, which is specified in Condition \ref{con:1}. Therefore, the induction assumption in \eqref{superh2} holds for $i=j+1$. In conclusion, by induction we prove that the cost function $C(K_i;Q_i,R_i)$ is upper bounded by $(\alpha F+2M)$ for all $i\ge0$.

	 It remains to prove that the stepsize $\eta$ of policy optimization specified in Condition \ref{con:1} satisfies the requirement in \eqref{etacond1} of  Lemma \ref{bdlm2}, which is concluded in the following lemma. %To this end, we prove that under Condition \ref{con:1}, $\eta$ is upper bounded by the uniform infimum of the right-hand side of \eqref{etacond1}.
	
\begin{lemma}\label{infcal}
We assume that $C(K_i;Q_i,R_i)\le\alpha F+2M$. For the first term on the right-hand side of \eqref{etacond1}, we have
\$
	&\left(\frac{\mu \cdot \sigma_{\min}(Q_i) }{C(K_i;Q_i,R_i)} \right)^2 \cdot 
	\frac{1}{ \|B\| \cdot \| \nabla_K C(K_i;Q_i,R_i)\| \cdot \bigl(1+\|A-B K_i\|\bigr) } \ge
	\frac{\alpha_Q^{3} \mu^{5/2}(\alpha F+2M)^{-7/2} }
	{ 
		\kappa_1^{1/2} \kappa_2\cdot \|B\| }.
	\$
	For the second term on the right-hand side of \eqref{etacond1}, we have
	\$
	\frac{\sigma_{\min}(Q_i)}{2 C(K_i;Q_i,R_i) \cdot \|R + B^\top P_{K_i}B\|} 
	\ge
	\frac{\alpha_Q}{2 \kappa_1 (\alpha F+2M)}.
	\$
\end{lemma}
\begin{proof}
See \S\ref{pinfcal} for a detailed proof.
\end{proof}
Given $C(K_i;Q_i,R_i)\le\alpha F+2M$ for all $i\ge0$, from \eqref{ineqs} we obtain 
\$
\|\Sigma_{K_i}\|\le(\alpha F+2M)/\alpha_Q
\$ for all $i\ge0$. Similar to \eqref{ineqs}, by the fact that $\Sigma_{K_i} \succeq \Sigma_0 \succeq \mu I$ and $R_i \succeq \alpha_R I$ for all $i\geq 0$, we have
\$
\alpha_R\mu\cdot \|K_i\|^2\le \langle K_i\Sigma_{K_i}K_i^\top, R_i \rangle \le C(K_i;Q_i,R_i),
\$
which implies 
\$
\|K_i\|^2\le(\alpha F+2M)/(\alpha_R\mu)
\$ for all $i\ge0$. Thus, we conclude the proof of Lemma \ref{lm:bc}.
	\end{proof}
	
\section{Proof of Lemma \ref{lm:dpf}}	\label{appp2}
To prove Lemma \ref{lm:dpf}, we lay out two auxiliary lemmas. The first lemma establishes an upper bound of the increment of the objective function $m(K,\theta)$ at each iteration in \eqref{pg} and \eqref{cg}.
\begin{lemma}\label{alm:pd}
Under Condition \ref{con:1} and assuming $\eta<1/(2\sigma_\theta \nu_V)$, we have
\$
	m(K_{i+1},\theta_{i+1})-m(K_{i},\theta_i) &\le (-1/\eta+\tau_{V})/2 \cdot\|K_{i+1}-K_i\|_{\tf}^2 \\
	&\qquad +(1/\lambda+\tau_{V}+\nu)/2\cdot\| \theta_{i+1}-\theta_{i} \|_\tf^2 
	+ (1/\lambda+\nu)/2\cdot\|\theta_{i}-\theta_{i-1}\|_{\tf}^2.
\$
\end{lemma}
\begin{proof}
See \S\ref{palm:pd} for a detailed proof.
\end{proof}
 The following lemma characterizes the evolution of gradient along the solution path.

\begin{lemma} \label{alm:pf}
	Under Condition \ref{con:1}, it holds that for all $i\ge0$,
\$
&1/2\cdot\|K_{i+1}-K_{i}\|_{\tf}^2-1/2 \cdot\|K_{i}-K_{i-1}\|_{\tf}^2 \\
&\quad \leq (\eta \lambda \tau_{V}^2 + 5/2\cdot \eta \nu_V \sigma_\theta) \cdot \|K_{i+1}-K_{i}\|_{\tf}^2 + (\eta\nu_V\sigma_\theta/2)\cdot\|K_{i}-K_{i-1}\|_{\tf}^2\\
&\quad\qquad -\eta/(2\lambda) \cdot \|\theta_{i+1}-\theta_{i}\|_{\tf}^2 + \bigl(\eta/(2\lambda) - \eta \gamma + \eta \lambda \nu^2\bigr) \cdot \|\theta_i-\theta_{i-1}\|_{\tf}^2.
\$
%&\quad\le \eta \lambda(\tau_{V}^2\|(K_{i+1}-K_i\|_{\tf}^2+\nu^2\|\theta_{i}-\theta_{i-1}\|^2)-\frac{\eta}{2\lambda}\cdot\|\theta_{i+1}-\theta_i\|^2+\frac{\eta}{2\lambda}\|\theta_{i}-\theta_{i-1}\|^2 -\eta\gamma \|\theta_i-\theta_{i-1}\|^2\\
%&\qquad  + 2\eta \nu_V \sigma_{\theta} \| K_{i+1}-K_i \|_{\tf}^2 +(\eta \nu_V \sigma_{\theta}/2)\cdot  (\|K_i-K_{i-1}\|_{\tf}^2+ \|K_{i+1}-K_i\|_{\tf}^2). 

\end{lemma}
\begin{proof}
See \S\ref{palm:pf} for a detailed proof.
\end{proof}

Based on Lemmas \ref{alm:pd} and \ref{alm:pf}, we now prove Lemma \ref{lm:dpf}.
	\begin{proof}
	For notational simplicity, we define 
	\$
	D_{i+1}=(1+\eta \nu_V \sigma_{\theta})/2 \cdot \|K_{i+1}-K_{i}\|_{\tf}^2
	+(\eta/\lambda-\eta\gamma+\eta\lambda \nu^2)/2\cdot
	\|\theta_{i+1}-\theta_i\|_{\tf}^2.
	\$ 
	By rearranging the inequality in Lemma \ref{alm:pf}, we obtain 
	\# \label{dd}
		D_{i+1}- D_i &\le-(\eta\gamma-\eta\lambda \nu^2)/2 \cdot \|\theta_i-\theta_{i-1}\|_{\tf}^2 
		+(\eta\lambda \tau_{V}^2+3\eta \nu_V \sigma_{\theta}) 
		\cdot\|K_{i+1}-K_i\|_{\tf}^2 \nonumber\\
		& \qquad-(\eta\gamma-\eta\lambda \nu^2)/2\cdot \|\theta_{i+1}-\theta_i\|_{\tf}^2.
	\#
	Meanwhile, by the definition of $P_i$ in \eqref{potdef}, we have 
	\$
	P_{i}=m(K_i,\theta_i)+s \cdot D_i
	\$ for some constant $s>0$. Combining \eqref{dd} and Lemma \ref{alm:pd}, we obtain
	\$
		P_{i+1}-P_i &\le 
		- \bigl(1/(2\eta)-\tau_{V}/2-s\cdot (\eta\lambda \tau_{V}^2+3\eta \nu_V \sigma_{\theta}) \bigr) \cdot\|K_{i+1}-K_i\|_{\tf}^2 \\
		&\qquad -\bigl(
		s \cdot (\eta\gamma-\eta\lambda \nu^2)/2-
		(1/\lambda+\tau_{V}+\nu)/2
		\bigr) \cdot\|\theta_{i+1}-\theta_i\|_{\tf}^2 \\
		&\qquad -\bigl(
		s\cdot (\eta\gamma-\eta\lambda \nu^2)/2-
		(1/\lambda+\nu)/2
		\bigr)  \cdot\|\theta_i-\theta_{i-1}\|_{\tf}^2,
	\$
	which implies \eqref{potd}. Now we choose a proper constant $s$ such that $\phi_1$, $\phi_2$, and $\phi_3$ are positive. Note that $\phi_1, \phi_2, \phi_3>0$ if and only if
	\#\label{eq:w001}
	1/(2\eta)- \tau_{V}/2>s\cdot (\eta\lambda \tau_{V}^2+3\eta \nu_V \sigma_{\theta}),\quad
	s \cdot (\eta\gamma-\eta\lambda \nu^2)/2 >
		(1/\lambda+\tau_{V}+\nu)/2.
	\#
	By Condition \ref{con:2}, it holds that $\eta\gamma-\eta\lambda\nu^2>0$. By rearranging the terms in \eqref{eq:w001}, we obtain
	\$
	\frac{1/\lambda+\tau_{V}+\nu}{\eta\gamma-\eta\lambda \nu^2}<s<\frac{1/\eta-\tau_{V}}{2\eta\lambda \tau_{V}^2+6\eta \nu_V \sigma_{\theta}}.
	\$
	 To ensure that such a constant $s$ exists, it suffices to choose stepsizes $\eta$ and $\lambda$ such that 
	\# \label{pf10}
	(1/\eta-\tau_{V})\cdot (\gamma-\lambda \nu^2) > (1/\lambda+\tau_{V}+\nu)\cdot (2\lambda \tau_{V}^2+6\nu_V\sigma_{\theta}). 
	\#
	Roughly speaking, by taking the leading terms, \eqref{pf10} requires $\eta/\lambda < \gamma/(6\nu_V\sigma_{\theta})$ for sufficiently small stepsizes $\eta$ and $\lambda$.
	To be more specific, there exist $\overline{\eta}$ and $\overline{\lambda}$ such that \eqref{pf10} holds for
	\$
	\eta/\lambda < \frac{\gamma}{7\nu_V\sigma_{\theta}} ,\quad \eta\le\overline{\eta},\quad \lambda\le\overline{\lambda}.
	\$
To this end, let
	\$
	\overline{\eta}=1/(100\tau_{V}),\quad 
	 \overline{\lambda}=\min\biggl\{ \frac{1}{100(\tau_{V}+\nu)},\frac{3\nu_V\sigma_\theta}{100\tau_{V}^2},\frac{\gamma}{100\nu^2} \biggr\}
	\$
	as in Condition \ref{con100}. Thus, we conclude the proof of Lemma \ref{lm:dpf}.
\end{proof}

\section{Proof of Theorem \ref{thm:lr}} \label{aplr}

\begin{proof}
	Since $\{(K_i,\theta_i)\}_{i=0}^\infty$ converges to $(K^*,\theta^*)$, by Lemmas \ref{qclm}, \ref{lm:lip3}, and \ref{lm:ksc}, there exists an iteration index $N>0$ and a neighborhood $\cB'$ of $(K^*,\theta^*)$, for which it holds for all $i>N$ that $(K_i,\theta_i)\in\cB'$ and 
	\begin{itemize}
	\item $C(K;\theta)$ is $(\alpha_R\mu)$-strongly convex with respect to $K$ within $\cB'$,
	\item $K^*(\theta)$ is $\nu_{K^*}$-Lipschitz continuous with respect to $\theta$ within $\cB'$,
	\item $m^*(\theta)$ is $\gamma$-strongly concave and $\nu_{m^*}$-smooth with respect to $\theta$ within $\cB'$.
	\end{itemize}
	Here recall that $K^*(\theta)$ and $m^*(\theta)$ are defined in \eqref{kstar}. Since $\{\theta_i\}_{i=0}^\infty$ converges to an interior point of $\Theta$, we omit the projection operator throughout the following proof. 
	
	For notational simplicity, we define 
	\#\label{eq:w0001}
	K^*_{i} = K^*(\theta_i)
	\# for all $i\ge0$ and the surrogate step 
	\#\label{eq:w999}
	\tilde{\theta}_{i+1}=\theta_i+\lambda \cdot \nabla m^*(\theta_{i}) =\theta_i + \lambda \cdot \bigl(V(K^*_{i})-\nabla\psi(\theta_i)\bigr).
	\#
	By the triangle inequality, we have
	\# \label{sc1}
	\| \theta_{i+1}-\theta^* \|_{\tf} \le \| \theta_{i+1}-\tilde{\theta}_{i+1} \|_{\tf}+
	\| \tilde{\theta}_{i+1}-\theta^* \|_{\tf}.
	\# 
	%Also, by the $\gamma$-strong concavity of $m^*(\theta)$ within $\cB'$ and following the standard convex optimization analysis \citep{bubeck2015convex}, we have that, under Condition \ref{con:w686},
	
	To upper bound $\| \tilde{\theta}_{i+1}-\theta^* \|_{\tf}$, we invoke the following lemma on the contraction of the gradient algorithm on a strongly convex and smooth objective function. 
		\begin{lemma}\label{921}
Let $g$ be a $\mu_{g}$-strongly convex and $\nu_g$-smooth function on $\real^{p}$, and $y^*\in\real^{p}$ be the unique minimizer of $g$. For $y, y'\in\real^{p}$ that satisfy
\$
y'=y-\eta_g\cdot\nabla g(y),
\$
where $0<\eta_g\le 2/(\mu_g+\nu_g)$, it holds that
\$
\|y'-y^*\|_2\le(1-\eta_g\mu_g) \cdot \|y-y^*\|_2.
\$
\end{lemma}
\begin{proof}
See the proof of Theorem 3.12 in \cite{bubeck2015convex} for a detailed proof.
\end{proof}
	
	Also, since $m^*(\theta)$ is $\gamma$-strongly concave and $\nu_{m^*}$-smooth within $\cB'$, by Lemma \ref{921} we have that, under Condition \ref{con:w686},
	\# \label{sc2}
	\| \tilde{\theta}_{i+1}-\theta^* \|_{\tf} \le 
	(1-\lambda \gamma) \cdot \| \theta_{i}-\theta^* \|_{\tf}.
	\#
	Meanwhile, the difference between the actual step $\theta_{i+1}$ in \eqref{cg} and the surrogate step $\tilde{\theta}_{i+1}$ in \eqref{eq:w999} is upper bounded by
	\# \label{sc3}
	\| \theta_{i+1}-\tilde{\theta}_{i+1} \|_{\tf} &=
	\lambda\cdot \|V(K^*_{i})-V(K_{i+1}) \|_{\tf} \notag
	\\& \leq \lambda \tau_{V} \cdot \|K^*_{i}-K_{i+1} \|_{\tf}
	\leq \lambda \tau_{V}\cdot (1- \eta \alpha_R\mu) \cdot \|K^*_{i}-K_{i} \|_{\tf}. 
	\#
	Here in the first inequality we invoke the Lipschitz continuity of $V(K)$, while in the second inequality we use the $\alpha_R \mu$-strong convexity of $C(K;\theta_i)$ around its minimizer $K^*_{i}$ and follow a similar proof of \eqref{sc2}. 
	Combining \eqref{sc1}, \eqref{sc2}, and \eqref{sc3}, we obtain
	\# \label{sc4}
	\| \theta_{i+1}-\theta^* \|_{\tf} \le
	(1-\lambda \gamma) \cdot \| \theta_{i}-\theta^* \|_{\tf}
	+\lambda \tau_{V}\cdot  (1- \eta \alpha_R \mu) \cdot \|K^*_{i}-K_{i} \|_{\tf}.
	\#
	
	To construct a linearly convergent potential function, we quantify $\|K_{i+1}-K^*_{i+1}\|_{\tf}$ in a manner similar to \eqref{sc4}. We invoke the triangle inequality and obtain
	\# 
	\|K_{i+1}-K^*_{i+1}\|_{\tf} &\le 
	\|K_{i+1}-K^*_{i}\|_{\tf}+\|K^*_{i+1}-K^*_{i}\|_{\tf} \nonumber \\
	& \le (1- \eta \alpha_R \mu) \cdot \|K^*_{i}-K_{i} \|_{\tf}+\|K^*_{i+1}-K^*_{i}\|_{\tf} \label{sc5}.
	\#
	Here in the second inequality we invoke the local $\alpha_R \mu$-strong convexity of $C(K;\theta_i)$ again as in \eqref{sc3}. To upper bound the difference between $K^*_{i+1}$	 and $K^*_{i}$, by \eqref{eq:w0001} and Lemma \ref{qclm} we obtain
	\# \label{sc6}
	\|K^*_{i+1}-K^*_{i}\|_{\tf} \le 
	\tau_{K^*}\cdot \|\theta_{i+1}-\theta_i \|_{\tf}. 
	\#	
	We then obtain an upper bound of $\|\theta_{i+1}-\theta_i \|_{\tf}$ in terms of $\| \theta_{i}-\theta^* \|_{\tf}$ and $ \|K^*_{i}-K_{i} \|_{\tf} $,
	\# \label{sc7}
	\|\theta_{i+1}-\theta_i \|_{\tf} &\le 
	\|\theta_{i+1}-\tilde{\theta}_{i+1} \|_{\tf}  +  \|\tilde{\theta}_{i+1}-\theta_i \|_{\tf} \nonumber \\
	&\le \lambda \tau_{V} \cdot (1- \eta \alpha_R \mu) \cdot \|K^*_{i}-K_{i} \|_{\tf} 
	+ \lambda \cdot \| \nabla m^*(\theta_i) \|_{\tf} \nonumber\\
	&= \lambda \tau_{V} \cdot (1- \eta \alpha_R \mu) \cdot \|K^*_{i}-K_{i} \|_{\tf} 
	+ \lambda \cdot \| \nabla m^*(\theta_i)- \nabla m^*(\theta^*) \|_{\tf} \nonumber\\
	&\le \lambda  \tau_{V} \cdot (1- \eta \alpha_R \mu) \cdot \|K^*_{i}-K_{i} \|_{\tf} 
	+ \lambda \nu_{m^*} \cdot  \| \theta_i- \theta^* \|_{\tf},
	\#
	where the second inequality follows from \eqref{sc3}, the equality follows from the fact that $\nabla m^*(\theta^*) = 0$, and the third inequality follows from Lemma \ref{lm:lip3}. Combining \eqref{sc5}, \eqref{sc6}, and \eqref{sc7}, we obtain
	\# 
	\|K_{i+1}-K^*_{i+1}\|_{\tf} = (1+\lambda \tau_{V} \tau_{K^*})\cdot (1- \eta \alpha_R \mu)\cdot  \|K^*_{i}-K_{i} \|_{\tf} 
	+ \lambda \nu_{m^*}\tau_{K^*} \cdot \| \theta_i- \theta^* \|_{\tf} \label{sc8}.
	\#

	Combining \eqref{sc4} and \eqref{sc8},  we have that for any constant $a>0$,
	\#\label{eq:wa666}
	&\| \theta_{i+1}-\theta^* \|_{\tf} +a \cdot \|K_{i+1}-K^*_{i+1}\|_{\tf} \notag\\
	&\quad\le 
	(1-\lambda \gamma + a\cdot \lambda  \nu_{m^*} \tau_{K^*} ) \cdot \| \theta_{i}-\theta^* \|_{\tf}\notag
	\\
	&\qquad\quad +\bigl(\lambda \tau_{V} \cdot (1- \eta \alpha_R \mu)+
	a\cdot (1+\lambda \tau_{V} \tau_{K^*})\cdot (1- \eta \alpha_R \mu) 
	\bigr)\cdot \|K^*_{i}-K_{i} \|_{\tf}.
	\#
	Since $\{\theta_i\}_{i=0}^\infty$ converges to the stationary point of $K^*(\theta)$, we are able to find a sufficiently small neighborhood $\cB_K$ of $\theta^*$ such that $\tau_{K^*}$ in Lemma \ref{qclm} is sufficient small and the following inequalities hold, 
		\# \label{con4}
		1-(1+\lambda \tau_{V}\tau_{K^*}) \cdot (1-\eta\alpha_R \mu)>0,\quad
	\frac{\lambda \tau_{V}\cdot (1-\eta\alpha_R \mu)}{1-(1+\lambda \tau_{V}\tau_{K^*}) \cdot (1-\eta\alpha_R \mu)} \leq \frac{\gamma}{3\tau_{K^*} \nu_{m^*}}.
	\#
	We set $a=\gamma/(3\tau_{K^*} \nu_{m^*})$ in the potential function $Z_i$ defined in \eqref{eq:wz1}. By \eqref{con4}, on the right-hand side of \eqref{eq:wa666} it holds that
	\$
	1-\lambda \gamma + a\cdot \lambda  \nu_{m^*} \tau_{K^*} < 1,\quad \lambda \tau_{V} \cdot (1- \eta \alpha_R \mu)+
	a\cdot (1+\lambda \tau_{V} \tau_{K^*})\cdot (1- \eta \alpha_R \mu) < a. 
	\$
	Then for $\upsilon$ defined in \eqref{qdef}, from \eqref{eq:wa666} we obtain
	\$
	Z_{i+1}\le \upsilon \cdot Z_i, \quad\text{where}~~ \upsilon \in (0,1),
	\$
	for all $i>N$, which concludes the proof of Theorem \ref{thm:lr}.
\end{proof}

\section{Proof of Auxiliary Lemmas in \S\ref{analysis} and \S\ref{proof}} \label{pal}
In this section, we present the proof of the auxiliary lemmas in \S\ref{analysis} and \S\ref{proof}.
\subsection{Proof of Lemma \ref{usp}}\label{pusp}
\begin{proof}
    Note that the proximal gradient $L(K,\theta)$ defined in \eqref{proxg} is continuous with respect to both $K$ and $\theta$, which are bounded along the solution path. The existence of a proximal stationary point is implied by the global convergence established in Theorem \ref{mainthm}. It remains to prove the uniqueness. We assume that $(K^{(1)},\theta^{(1)})$ and $(K^{(2)},\theta^{(2)})$ are two distinct proximal stationary points. By the fact that $m(K,\theta)$ is strongly concave with respect to $\theta$ and Lemma \ref{lm:stp}, when either $K$ or $\theta$ are fixed, $m(K,\theta)$ has a unique maximizer or minimizer in terms of $\theta$ or $K$, respectively. Also, the proximal gradient being zero implies that $(K^{(1)},\theta^{(1)})$ and $(K^{(2)},\theta^{(2)})$ are both saddle points of $m(K,\theta)$. Thus, by the optimality of $K^{(1)}$, $K^{(2)}$, $\theta^{(1)}$, and $\theta^{(2)}$, we have 
	\$m(K^{(1)},\theta^{(1)}) \le m(K^{(2)},\theta^{(1)}) \le m(K^{(2)},\theta^{(2)}) \le m(K^{(1)},\theta^{(2)})\le m(K^{(1)},\theta^{(1)}), \$ which implies 
	\#\label{eq:wa01}
	m(K^{(1)},\theta^{(1)})=m(K^{(2)},\theta^{(2)}),\quad m(K^{(1)},\theta^{(1)})=m(K^{(2)},\theta^{(1)}).
	\# By the uniqueness of the maximizer and minimizer of $m(K, \theta)$ in terms of $\theta$ and $K$, respectively, we obtain $K^{(1)}=K^{(2)}$ and $\theta^{(1)}=\theta^{(2)}$ from \eqref{eq:wa01}, which concludes the proof of Lemma \ref{usp}.
\end{proof}

\subsection{Proof of Lemma \ref{lm:stp}}\label{plm:stp}
\begin{proof}
	By Lemma 12 in \cite{fazel2018global}, which is presented in \S\ref{sec:aux}, we have
	\#\label{eq:wzr}
	C(K;\theta)-C(K';\theta)
	&=\tr\bigl(
	\Sigma_{K}(K-K')^\top (R+B^\top P_{K'}B)(K-K')\bigr) 
	\#
	when $\nabla_K C(K';\theta)=0$. The proof immediately follows from the fact that $\Sigma_{K}$ and $R+B^\top P_{K'}B$ are positive definite, which implies the right-hand side of \eqref{eq:wzr} is nonpositive if and only if $K=K'$.
\end{proof}

\subsection{Proof of Lemma \ref{lm:lip}}\label{plm:lip}
\begin{proof}
We first establish the compactness of the set 
\#\label{eq:wa2}
\cT^S=\bigl\{T\in\real^{d\times d} \,|\, T=A-BK, \|\Sigma_K\|\le S\bigr\}.
\# 
To this end, it remains to prove that $\cT^S$ is closed and bounded.

To prove that $\cT^S$ is closed, note that for any stabilizing policy $K$, we have
\#\label{eq:wa3}
 \tr(\Sigma_K)&=\sum_{t=0}^{\infty}\tr\bigl(T^t\Sigma_0(T^\top)^t\bigr)
=	\sum_{t=0}^{\infty}\tr\bigl((T^\top)^tT^t\Sigma_0\bigr)	\notag\\
&\ge \mu \sum_{t=0}^{\infty}\tr\big((T^\top)^tT^t\big) 
= \mu \sum_{t=0}^{\infty}\|T^t\|_{\tf}^2  
\ge \mu \sum_{t=0}^{\infty}\rho(T^t)^2  
= \mu \sum_{t=0}^{\infty}\rho(T)^{2t},
\#
where the first equality follows from 
\#\label{eq:wa4}
\Sigma_{K}=\E\bigl[{\textstyle\sum_{t=0}^{\infty}} x_t x_t^\top \,|\, x_t = T^t x_0\bigr],\quad\Sigma_0=\E[x_0 x_0^\top], 
\#
the first inequality follows from the fact that $\Sigma_0 \succeq \mu I$ and $(T^\top)^tT^t \succeq 0$, and the second inequality holds since $\|T^t\|_{\tf} \geq \rho(T^t)$, where $\rho(T^t)$ denotes the spectral radius. Hence, the condition $\|\Sigma_K\|\le S$ in \eqref{eq:wa2} implies that there exists a constant $\delta>0$ that depends on $S$ such that $\rho(T)\le1-\delta$, since otherwise the right-hand side of \eqref{eq:wa3} tends to infinity. Meanwhile, since the function 
\$
\|\Sigma_K\|=\| \textstyle{\sum_{t=0}^{\infty}T^t} \Sigma_0(T^\top)^t \|
\$
 is continuous with respect to $T$ that satisfies
\$
T \in \bar{\cT}^S = \bigl\{T\in\real^{d\times d}\,|\,T=A-BK,\rho(T)\le1-\delta\bigr\},
\$ where $\bar{\cT}^S$ is a closed set, we conclude that $\cT^S \subseteq \bar{\cT}^S$ is a closed set. 
 
Now we prove that $\cT^S$ is bounded. Since we have
$$
\Sigma_{K}=\sum_{t=0}^{\infty}T^t\Sigma_0(T^\top)^t\succeq T\Sigma_0T^\top\succeq0
$$
and $\Sigma_0 \succeq \mu I$, we obtain 
\#\label{eq:wa10}
\|\Sigma_K\|>\|T\Sigma_0T^\top\|\ge \mu \cdot \|T\|^2.
\# Thus, we obtain $\|T\|^2\le S/\mu$, which implies that $\cT^S$ is bounded. 

By \eqref{eq:wa3} and \eqref{eq:wa4}, each entry $(\Sigma_K)_{i,j}$ of $\Sigma_K$ is a power series of the entries of $T$, since we have  
\#\label{eq:wa5}
\Sigma_K&=\sum_{t=0}^{\infty} T^t\Sigma_0(T^\top)^t,
\#
Hence, $(\Sigma_K)_{i,j}$ is an analytic function with respect to the entries of $T$ as long as such a power series is~convergent. Moreover, all the derivatives of $(\Sigma_K)_{i,j}$ with respect to $T$ are also analytic functions over $\cT^{S}$. Hence, all the derivatives of $(\Sigma_K)_{i,j}$ are continuous functions on $\cT^{S}$. Since $\cT^{S}$ is compact, such derivatives are bounded.

%Note that for each entry $(i,j)$ of $\Sigma_K$, it a power series of elements in $T$, which is an analytic function with respect to entries of $T$ as long as the series is convergent. So $g_{i,j}(T)=(\Sigma_K)_{i,j}$ is also analytic because square root function is an analytic function and compositions of analytic functions are analytic. Then all derivatives of $(\Sigma_K)_{i,j}$ w.r.t $T$'s elements are also analytic functions on $\cT^{\S}$. This is because term by term differentiation on a power series does not change its radius of convergence. Therefore all derivatives of $g_{i,j}(T)=(\Sigma_K)_{i,j}$ are continuous function on $\cT^{S}$. And thus they are bounded since $T$ belongs to a compact set.

Since $T = A-BK$, the derivatives of $T$ with respect to $K$ are independent of $K$. Therefore, by \eqref{eq:wa5} and chain rule 
 we have that all the derivatives of $\Sigma_K$ with respect to the entries of $K$ are also continuous and bounded, as long as 
\$
K \in \bigl\{K\in \real^{k\times d} \,|\,\|\Sigma_K\|\le S\bigr\},
\$
which corresponds to $T \in \cT^S$. Thus, we obtain the desired Lipschitz continuity and smoothness of $\Sigma_K$, which concludes the proof of Lemma \ref{lm:lip}.
\end{proof}

\subsection{Proof of Lemma \ref{qclm}}\label{pqclm}
	\begin{proof}
	According to the analytic implicit function theorem \citep{fritzsche2002holomorphic}, Condition \ref{con:2} implies the Lipschitz continuity and smoothness of $P^*(Q,R)$ within a neighborhood of $(Q^*,R^*)$. Thus, the optimal policy  as a function with respect to the cost parameter $\theta$, which takes the form
\$K^*(\theta)=\bigl(B^\top P^*(Q,R)B+R\bigr)^{-1}B^\top P^*(Q,R)A,\$ 
inherits the Lipschitz continuity and smoothness of $P^*(Q,R)$, since we have 
\$
B^\top P^*(Q,R)B+R\succeq R \succeq \alpha_R I
\$ where $P^*(Q,R) \succeq 0$. Thus, we conclude the proof of Lemma \ref{qclm}.
	\end{proof}	
	
\subsection{Proof of Lemma \ref{lm:lip3}}\label{plm:lip3}
	\begin{proof}
	Recall that
		\$
K^*(\theta)=\argmin_{K\in\cK}C(K;\theta),\quad m^*(\theta)=m\bigl(K^*(\theta), \theta\bigr) =C\bigl(K^*(\theta);\theta\bigr)-C(K_{\text{E}};\theta)-\psi(\theta).
		\$
		By the Lipschitz continuity and smoothness of $m(K,\theta)$, which follow from Lemma \ref{lm:lip1}, the smoothness of $\psi(\theta)$, and the boundedness of $\Theta$, 
		 the smoothness of $m^*(\theta)$ immediately follows from Lemma \ref{qclm}.  To prove that $m^*(\theta)$ is $\gamma$-strongly concave, we first prove that $C(K^*(\theta); \theta)$ is concave, which follows from the fact that  
		 \$
		 C\bigl(K^*(\theta);\theta\bigr)=\sup_{K}C(K;\theta),
		 \$
		 where as defined in \eqref{inner2}, $C(K;\theta)$ is a linear function with respect to $\theta$. Therefore, $m(K^*(\theta),\theta)$ is  $\gamma$-strongly concave, since $\psi(\theta)$ is $\gamma$-strongly convex. Thus, we conclude the proof of Lemma \ref{lm:lip3}. 
\end{proof} 
	
	\subsection{Proof of Lemma \ref{lm:ksc}}\label{plm:ksc}
\begin{proof}
	By Lemma 12 in \cite{fazel2018global}, which is presented in \S\ref{sec:aux}, we have
	\#\label{eq:w007}
	C(K;\theta)-C\bigl(K^*(\theta);\theta\bigr)
	&=\tr\Bigl(
	\Sigma_{K}\bigl(K-K^*(\theta)\bigr)^\top (R+B^\top P_{K^*(\theta)}B)\bigl(K-K^*(\theta)\bigr)\Bigr) \notag\\
	&\ge \mu \cdot \tr\Bigl( \bigl(K-K^*(\theta)\bigr)^\top (R+B^\top P_{K^*(\theta)}B)\bigl(K-K^*(\theta)\bigr) \Bigr) \notag\\
	&= \mu \cdot \tr\Bigl( (R+B^\top P_{K^*(\theta)}B)\bigl(K-K^*(\theta)\bigr)\bigl(K-K^*(\theta)\bigr)^\top\Bigr) \notag\\
	&\ge \alpha_R \mu \cdot \tr\Bigl( \bigl(K-K^*(\theta)\bigr)\bigl(K-K^*(\theta)\bigr)^\top\Bigr) \notag\\
	&= \alpha_R \mu \cdot \| K-K^*(\theta) \|^2_{\tf}.
	\#
Here the first inequality follows from the fact that $\Sigma_K \succeq \Sigma_0 \succeq \mu I$, and the second inequality follows from the fact that $B^\top P^*(Q,R)B+R\succeq\alpha_R I$.

	By Lemma \ref{lm:lip1}, we have that $C(K;\theta)$ is smooth. Also, recall that $K^*(\theta)$ is a stationary point of $K$ for a fixed $\theta$. Hence, \eqref{eq:w007} implies that, for any $\theta\in\Theta$, the Hessian matrix $H \in \RR^{dk\times dk}$ of $C(K;\theta)$ with respect to $\vec(K)$ at $K^*$ satisfies $H\succeq2\alpha_R \mu \cdot I$. Therefore, $C(K;\theta)$ is $(\alpha_R \mu)$-strongly convex within a neighborhood of $K^*(\theta)$, which concludes the proof of Lemma \ref{lm:ksc}.
\end{proof}		
	
\section{Proof of Auxiliary Lemmas in Appendix}	
In this section, we lay out the proof of the auxiliary lemmas presented in the appendix. 

\subsection{Proof of Lemma \ref{infcal}}\label{pinfcal}
\begin{proof}
	For notational simplicity, we drop the subscript $i$ throughout the following proof. First, by the assumption of Lemma \ref{infcal} and the definition of $P_K$ in \eqref{eq:wa6}, we have
	\# \label{b112}
	&\alpha F+2M \ge C(K, \theta) = \langle \Sigma_0, P_K \rangle \ge \mu\cdot \|P_K\|
	\#
	along the solution path, where the second inequality holds because $\Sigma_0 \succeq \mu I$. By Lemma 22 in \cite{fazel2018global}, which is presented in \S\ref{sec:aux}, we have
	\#\label{eq:wa8}
	&\|\nabla_K C(K;Q,R)\| \notag\\
	&\quad\le
	\frac{C(K;Q,R)}{\mu^{1/2} \cdot \sigma_{\text{min}}(Q)} \cdot 
	\biggl(
	\|R+B^\top P_K B \| \cdot \Bigl(C(K;Q,R)-C\bigl(K^*(Q,R);Q,R\bigr)\Bigr)\biggr)^{1/2}. 
\#
Since the cost function is nonnegative, we have
\$
C(K;Q,R)-C\bigl( K^*(Q,R);Q,R \bigr)\le C(K;Q,R)\le \alpha F+2M.
\$
Recall that $\kappa_1$ and $\kappa_2$ are defined in \eqref{eq:wa9}. From \eqref{b112} we obtain the following upper bound of $\|R+B^\top P_K B \|$,
\# \label{b113}
\|R+B^\top P_K B \|	\le \|R\|+\|P_K\|\cdot \|B\|^2\le\beta_R+(\alpha F+2M)/\mu \cdot \|B\|^2=\kappa_1,
\#
where we use the fact that $R \preceq \beta_R I$.  Thus, from \eqref{eq:wa8} we obtain
\$
\|\nabla_K C(K;Q,R)\| 
	&\le \frac{\alpha F+2M}{\alpha_Q \mu^{1/2}} \cdot 
	\bigl(\|R+B^\top P_K B \| \cdot (\alpha F+2M)\bigr)^{1/2} \\
	&\le (\alpha F+2M)^{3/2}\alpha_Q^{-1}\mu^{-1/2}
	\kappa_1^{1/2}.
	\$
	Recall that in the proof of Lemma \ref{lm:lip}, we obtain in \eqref{eq:wa10} that 
	\$
	\|A-BK\|\le \mu^{-1/2} \cdot \|\Sigma_K\|^{1/2}\le
	(\mu\alpha_Q)^{-1/2} (\alpha F+2M)^{1/2}=\kappa_2-1.
	\$
	Therefore, for the first term on right-hand side of \eqref{etacond1}, we have 
	\$
	\biggl(\frac{\mu \cdot \sigma_{\min}(Q) }{C(K;Q,R)} \biggr)^2 \cdot
	\frac{1}{ \|B\| \cdot \| \nabla_K C(K;Q,R)\| \cdot \bigl(1+\|A-B K\|\bigr) } 
	\ge
	\frac{\alpha_Q^{3} \mu^{2.5}(\alpha F+2M)^{-3.5} }
	{ 
		\kappa_1^{1/2} \kappa_2 \cdot \|B\| }
	\$
	Meanwhile, the second term on the right-hand side of \eqref{etacond1} is similarly obtained using the fact that $\|R+B^\top P_K B \|	\le \kappa_1$. Thus, we conclude the proof of Lemma \ref{infcal}.
\end{proof}	

\subsection{Proof of Lemma \ref{alm:pd}}\label{palm:pd}
\begin{proof}
	For the update of policy in \eqref{pg}, from \eqref{svform} we have
	\#\label{eq:wang}
		m(K_{i+1}, \theta_i)-m(K_i, \theta_i)  
		&= \bigl \langle V(K_{i+1})-V(K_i),\theta_i \bigr\rangle= {\textstyle \sum_{j,\ell=1}^{d+k}}  (\theta_i)_{j,\ell} \cdot \bigl(V_{j,\ell} (K_{i+1})-V_{j,\ell} (K_i)\bigr),
	\#
	where $V_{j,\ell}(K)$ denotes the $(j,\ell)$-th entry of $V(K) \in \RR^{(d+k)\times (d+k)}$ and $(\theta_i)_{j,\ell}$ denotes the $(j, \ell)$-th entry of $\theta_i \in \RR^{(d+k)\times (d+k)}$, which is defined in \eqref{innerform}. Let $\tilde{K}_i^{j,\ell}$ be the interpolation between $K_i$ and $K_{i+1}$ in the mean value theorem such that for each $(j,\ell)$,
	\# \label{ktilde}
	\bigl\langle \nabla V_{j,\ell}(\tilde{K}_i^{j,\ell}), K_{i+1}-K_i \bigr\rangle=V_{j,\ell}(K_{i+1})-V_{j,\ell}(K_{i}).
	\#
	Then from \eqref{eq:wang} we obtain
	\# \label{cai813}
		& m(K_{i+1}, \theta_i)-m(K_i, \theta_i) =   {\textstyle \sum_{j,\ell=1}^{d+k}}  (\theta_i)_{j,\ell} \cdot  \big \langle \nabla V(\tilde{K}_i^{j,\ell}), K_{i+1}-K_i   \bigr \rangle \\
		&\quad= \bigl\langle  {\textstyle \sum_{j,\ell=1}^{d+k}}  (\theta_i)_{j,\ell} \cdot \nabla V_{j,\ell}(K_{i})  ,K_{i+1}-K_{i} \bigr\rangle
		+{\textstyle \sum_{j,\ell=1}^{d+k}}  (\theta_i)_{j,\ell} \cdot  \big \langle \nabla V_{j,\ell}(\tilde{K}_i^{j,\ell})-\nabla V_{j,\ell}(K_{i}), K_{i+1}-K_i   \bigr \rangle.  \notag
	\#
	By the Cauchy-Schwarz inequality, we have
	\#\label{cai809}
		\big \langle \nabla V_{j,\ell}(\tilde{K}_i^{j,\ell})-\nabla V_{j,\ell}(K_{i}), K_{i+1}-K_i   \bigr \rangle
		&\le \|\nabla V_{j,\ell}(\tilde{K}_i^{j,\ell})-\nabla V_{j,\ell}(K_{i})\|_{\tf} \cdot\|K_{i+1}-K_i \|_{\tf}  \notag \\
		& \le \nu_V/(d+k)\cdot  \|\tilde{K}_i^{j,\ell}-K_{i} \|_{\tf} \cdot\|K_{i+1}-K_i \|_{\tf}.
	\#
	where the second inequality follows from the $\nu_V$-smoothness of $V(K)$ established in Lemma \ref{lm:lip1}. 
	Thus, we obtain
	\# \label{cai812}
	&{\textstyle \sum_{j,\ell=1}^{d+k}}  (\theta_i)_{j,\ell} \cdot  \big \langle \nabla V_{j,\ell}(\tilde{K}_i^{j,\ell})-\nabla V_{j,\ell}(K_{i}), K_{i+1}-K_i   \bigr \rangle  \notag\\
	&\quad\le {\textstyle \sum_{j,\ell=1}^{d+k}}  \bigl|(\theta_i)_{j,\ell}  \bigr| \cdot \nu_V/(d+k)\cdot  \|\tilde{K}_i^{j,\ell}-K_{i} \|_{\tf} \cdot\|K_{i+1}-K_i \|_{\tf} \notag \\
	&\quad \le \nu_V\cdot \|K_{i+1}-K_i \|_{\tf} \cdot \|\theta_i\|_{\tf}  \cdot\bigl( {\textstyle \sum_{j,\ell=1}^{d+k}} \| \tilde{K}_i^{j,\ell}-K_{i} \|_{\tf}^2/(d+k)^2  \bigr)^{1/2} \notag\\
	&\quad \le \sigma_{\theta}\nu_V \cdot \|K_{i+1}-K_i \|_{\tf} ^2.
	\# 
	Here in the first inequality we plug in \eqref{cai809}, in the second inequality we use the Cauchy-Schwarz inequality and the definition of  $\sigma_{\theta}$ in \eqref{nota1}, and in the third inequality we use the fact that 
	\$
	\|\tilde{K}^{j,\ell}_i-K_i \|_{\tf}\le\|K_{i+1}-K_i \|_{\tf}.
	\$
	Plugging \eqref{cai812} into \eqref{cai813}, we obtain
	\#\label{eq:w6}
	m(K_{i+1}, \theta_i)-m(K_i, \theta_i)\le\bigl\langle  {\textstyle \sum_{j,\ell=1}^{d+k}}  (\theta_i)_{j,\ell} \cdot \nabla V_{j,\ell}(K_{i})  ,K_{i+1}-K_{i} \bigr\rangle
	+\sigma_{\theta}\nu_V \cdot \|K_{i+1}-K_i \|_{\tf} ^2.
	\#
	Meanwhile,  by \eqref{pg} and \eqref{svform} we have 
	\$
	K_{i+1}-K_i = -\eta\cdot {\textstyle \sum_{j,\ell=1}^{d+k}} (\theta_i)_{j,\ell} \cdot \nabla V_{j,\ell}(K_{i}).
	\$
	Hence, we have
	\$
	\bigl\langle  {\textstyle \sum_{j,\ell=1}^{d+k}}  (\theta_i)_{j,\ell} \cdot \nabla V_{j,\ell}(K_{i})  ,K_{i+1}-K_{i} \bigr\rangle = -1/\eta\cdot\|K_{i+1}-K_i\|^2.
	\$
	Thus,  by \eqref{eq:w6} we obtain the following upper bound of $m(K_{i+1},\theta_i)-m(K_i,\theta_i)$,
	\#	
		m(K_{i+1},\theta_i)-m(K_i,\theta_i) 
		&\le - 1/\eta \cdot \|K_{i+1}-K_i\|^2 + \sigma_\theta \nu_V \cdot \|K_{i+1}-K_i\|^2 \notag\\
		&\le - 1/(2\eta) \cdot \|K_{i+1}-K_i\|^2 \label{pd7},
	\#
	where the second inequality follows from the assumption that $\eta<1/(2\sigma_\theta \nu_V)$ of Lemma \ref{alm:pd}.
	
	For the update of cost parameter in \eqref{cg}, the increment of $m(K,\theta)$ takes the form
	\$
	m(K_{i+1},\theta_{i+1})-m(K_{i+1},\theta_i)
		= \bigl\langle V(K_{i+1})-V(K_{\text{E}}), \theta_{i+1}-\theta_i \bigr\rangle 
		- \bigl(\psi(\theta_{i+1}) -\psi(\theta_i)\bigr). 
	\$
	By the convexity of $\psi(\cdot)$, we have
	\$
	\psi(\theta_{i+1}) -\psi(\theta_i)\ge \bigl\langle \nabla\psi(\theta_i),\theta_{i+1}-\theta_i \bigr\rangle.
	\$
	Hence, we obtain
	\# \label{944}
	 m(K_{i+1},\theta_{i+1})-m(K_{i+1},\theta_i)\le \bigl\langle V(K_{i+1})-V(K_{\text{E}})-\nabla\psi(\theta_i),\theta_{i+1}-\theta_i \bigr\rangle.
	\#
	To further obtain an upper bound of the right-hand side of \eqref{944}, we first define	
	\#\label{pd1}
	\varepsilon_{i+1}=\theta_{i+1}-\bigl(\theta_i+\lambda\cdot \nabla_\theta m(K_{i+1},\theta_i)\bigr)
	=\theta_{i+1}-\Bigl(\theta_i+\lambda \cdot \bigl(V(K_{i+1})-V(K_{\text{E}})-\nabla\psi(\theta_i)\bigr)\Bigr). 
	\#
	Then we have 
	\#\label{eq:w99}
	V(K_{i+1})-V(K_{\text{E}})-\nabla\psi(\theta_i)=(\theta_{i+1}-\theta_i-\varepsilon_{i+1})/\lambda.
	\# 
	Plugging \eqref{eq:w99} into \eqref{944} and invoking the fact that 
	\$
	\langle \varepsilon_i,\theta_{i+1}-\theta_{i} \rangle\le 0,
	\$ which follows from the convexity of $\Theta$ and the fact that $\theta_{i+1} = \Pi_\Theta[\theta_i+\lambda\cdot \nabla_\theta m(K_{i+1},\theta_i)]$, we obtain
	\# \label{pd5}
	m(K_{i+1},\theta_{i+1})-m(K_{i+1},\theta_i)
	&= 1/\lambda\cdot \langle \theta_{i+1}-\theta_i-\varepsilon_{i+1},\theta_{i+1}-\theta_i \rangle \nonumber\\
	&\le 1/\lambda \cdot \|\theta_{i+1}-\theta_{i}\|^2
	+ 1/\lambda \cdot \langle \varepsilon_i-\varepsilon_{i+1}, \theta_{i+1}-\theta_{i} \rangle.
	\#
	By the definition of $\varepsilon_i$ in \eqref{pd1}, for the iteration indices $i$ and $i+1$ we have
	\#
		\varepsilon_{i+1}&=\theta_{i+1}-\bigl(\theta_i+\lambda \cdot\nabla_{\theta} m(K_{i+1},\theta_i)\bigr), \label{pd2}\\
		\varepsilon_{i}&=\theta_{i}-\bigl(\theta_{i-1}+\lambda\cdot \nabla_{\theta} m(K_{i},\theta_{i-1})\bigr). \label{pd3}
	\#
	We subtract \eqref{pd3} from \eqref{pd2} and obtain
	\# \label{pd4}
		 \varepsilon_{i+1}-\varepsilon_i = (\theta_{i+1}-\theta_i)-(\theta_i-\theta_{i-1})-\lambda\cdot\bigl(\nabla_{\theta} m(K_{i+1},\theta_i)-\nabla_{\theta} m(K_{i},\theta_{i-1})\bigr). 
	\#
	Plugging \eqref{pd4} into the second term on the right-hand side of \eqref{pd5}, we obtain
	\# \label{1009}
	&1/\lambda\cdot \langle \varepsilon_i-\varepsilon_{i+1}, \theta_{i+1}-\theta_{i} \rangle \\
	&\quad = -1/\lambda \cdot \|\theta_{i+1}-\theta_{i}\|_{\tf}^2 
	+ \underbrace{1/\lambda \cdot  \langle \theta_i-\theta_{i-1},  \theta_{i+1}-\theta_{i} \rangle}_{\dr (i)}
	+ \underbrace{\bigl\langle \nabla_{\theta} m(K_{i+1},\theta_i)-\nabla_{\theta} m(K_{i},\theta_{i-1}),  \theta_{i+1}-\theta_{i} \bigr\rangle}_{\dr (ii)}.\notag
	\#
	For term (i) in \eqref{pd4}, we apply the Cauchy-Schwarz inequality to obtain an upper bound,
	\#  \label{10081}
	1/\lambda \cdot  \langle \theta_i-\theta_{i-1},  \theta_{i+1}-\theta_{i} \rangle &\le 
	1/\lambda \cdot \| \theta_i-\theta_{i-1}\|_{\tf} \cdot \| \theta_{i+1}-\theta_{i}\|_{\tf} \nonumber\\
	&\le
	 1/(2\lambda)\cdot \| \theta_i-\theta_{i-1} \|_{\tf}^2
	+ 1/(2\lambda)\cdot \| \theta_{i+1}-\theta_{i} \|_{\tf}^2.
	\# 
	For term (ii) in \eqref{pd4}, we apply the Lipschitz continuity of $\nabla_\theta m(K,\theta)$ with respect to $K$ and $\theta$ to obtain
	\# \label{10082}
	&\bigl\langle \nabla_{\theta} m(K_{i+1},\theta_i)-\nabla_{\theta} m(K_{i},\theta_{i-1}),  \theta_{i+1}-\theta_{i} \bigr \rangle\nonumber \\
	&\quad= \bigl \langle \nabla_{\theta} m(K_{i+1},\theta_i)-\nabla_{\theta} m(K_{i},\theta_{i}),  \theta_{i+1}-\theta_{i} \bigr\rangle+\bigl\langle \nabla_{\theta} m(K_{i},\theta_i)-\nabla_{\theta} m(K_{i},\theta_{i-1}),  \theta_{i+1}-\theta_{i} \bigr\rangle\nonumber\\
	&\quad\le \tau_{V}\cdot\|K_{i+1}-K_{i}\|_{\tf}\cdot\|\theta_{i+1}-\theta_i\|_{\tf}+ \nu \cdot\|\theta_{i}-\theta_{i-1}\|_{\tf}\cdot\|\theta_{i+1}-\theta_i\|_{\tf} \nonumber\\
	&\quad\le \tau_{V}/2 \cdot\|K_{i+1}-K_{i}\|_{\tf}^2+ \tau_{V}/2\cdot\|\theta_{i+1}-\theta_i\|_{\tf}^2+ \nu/2 \cdot\|\theta_{i}-\theta_{i-1}\|_{\tf}^2 + \nu/2\cdot \|\theta_{i+1}-\theta_i\|_{\tf}^2.
	\#
	The first inequality follows from the $\tau_V$-Lipschitz continuity of $V(K)$ in Lemma \ref{lm:lip1}, the $\nu$-smoothness of $\psi(\cdot)$, and the fact that
	\$
	\nabla_{\theta} m(K_{i+1},\theta_i)-\nabla_{\theta} m(K_{i},\theta_i)&= \bigl( V(K_{i+1})-V(K_{\text{E}})-\nabla_{\theta}\psi(\theta_i)   \bigr)
	- \bigl( V(K_{i})-V(K_{\text{E}})-\nabla_{\theta}\psi(\theta_i)   \bigr) \\
	&=V(K_{i+1})-V(K_{i}),
	\$
	which is implied by \eqref{eq:w8}, \eqref{eq:w9} and \eqref{innerform}.
	By plugging \eqref{10081} and \eqref{10082} into \eqref{1009}, we obtain 
	\#\label{pd6}
	&1/\lambda \cdot \langle \varepsilon_i-\varepsilon_{i+1}, \theta_{i+1}-\theta_{i} \rangle \\
	&\quad\le -(1/\lambda-\tau_{V}-\nu)/2\cdot\| \theta_{i+1}-\theta_{i} \|_{\tf}^2
	+ (1/\lambda+\nu)/2\cdot\|\theta_{i}-\theta_{i-1}\|_{\tf}^2
	+ \tau_{V}/2\cdot \|K_{i+1}-K_{i}\|_{\tf}^2. \nonumber
	\#
	Further plugging \eqref{pd6} into \eqref{pd5}, we obtain 
	\# \label{pd8}
		&m(K_{i+1},\theta_{i+1})-m(K_{i+1},\theta_i) \nonumber \\
		&\quad\le (1/\lambda+\tau_{V}+\nu)/2\cdot\| \theta_{i+1}-\theta_{i} \|_{\tf}^2
		+ (1/\lambda+\nu)/2\cdot\|\theta_{i}-\theta_{i-1}\|_{\tf}^2
		+ \tau_{V}/2\cdot\|K_{i+1}-K_{i}\|_{\tf}^2. 
	\#
	
	Finally, by combining \eqref{pd7} and \eqref{pd8}, which correspondingly characterize the update of policy and cost parameter, we obtain
	\$
		&m(K_{i+1},\theta_{i+1})-m(K_{i},\theta_i) \\
		&\quad \le (-1/\eta+\tau_{V})/2\cdot \|K_{i+1}-K_i\|_{\tf}^2+(1/\lambda+\tau_{V}+\nu)/2\cdot\| \theta_{i+1}-\theta_{i} \|_{\tf}^2
		+ (1/\lambda+\nu)/2\cdot\|\theta_{i}-\theta_{i-1}\|_{\tf}^2,
	\$
	which concludes the proof of Lemma \ref{alm:pd}.
\end{proof}

\subsection{Proof of Lemma \ref{alm:pf}}\label{palm:pf}
We prove Lemma \ref{alm:pf} based on the following auxiliary lemmas. The first lemma characterizes the update of policy in \eqref{pg}. 
	
\begin{lemma}\label{sub1}
	Under Condition \ref{con:1}, we have
	\$
	&\bigl\langle (K_{i+1}-K_{i})-(K_{i}-K_{i-1}), K_{i+1}-K_i \bigr\rangle \\
	&\quad \le 
	-\eta/\lambda \cdot \bigl\langle 
	(\theta_{i+1}-\theta_{i})-(\theta_{i}-\theta_{i-1})-(\varepsilon_{i+1}-\varepsilon_i), \theta_{i}-\theta_{i-1}
	\bigr\rangle -\eta\gamma\cdot \|\theta_i-\theta_{i-1}\|_{\tf}^2 \nonumber  \\
	&
	\quad\qquad+2 \eta \nu_V \sigma_\theta\cdot \| K_{i+1}-K_i \|_{\tf}^2+\eta \nu_V\sigma_{\theta}/2\cdot \bigl(\|K_i-K_{i-1}\|_{\tf}^2+ \|K_{i+1}-K_i\|_{\tf}^2 \bigr)
	\$
	for all $i\ge0$.
\end{lemma}
\begin{proof}
See \S\ref{psub1} for a detailed proof.
\end{proof}
The next lemma characterizes the update of cost parameter in \eqref{cg}. For notational simplicity, we define 
 \#\label{501} U_{i+1}=(\theta_{i+1}-\theta_{i})-(\theta_{i}-\theta_{i-1}).\#

\begin{lemma}\label{sub2}
	Under Condition \ref{con:1}, we have
	\# \label{500}
	-\eta/\lambda\cdot \bigl\langle U_{i+1}-(\varepsilon_{i+1}-\varepsilon_i), \theta_{i}-\theta_{i-1} \bigr\rangle & \le \eta\lambda\cdot\bigl(
	\tau_{V}^2 \cdot\|K_{i+1}-K_i\|_{\tf}^2+\nu^2\cdot\|\theta_{i}-\theta_{i-1}\|_{\tf}^2 \bigr)  \notag \\
	& \qquad - \eta/(2\lambda)\cdot\|\theta_{i+1}-\theta_i\|_{\tf}^2 + \eta/(2\lambda)\cdot\|\theta_{i}-\theta_{i-1}\|_{\tf}^2.
	\#
\end{lemma}
\begin{proof}
See \S\ref{psub2} for a detailed proof.
\end{proof}

With Lemmas \ref{sub1} and \ref{sub2}, we now prove Lemma \ref{alm:pf}.
\begin{proof}
Note that we have
	\# \label{pf9}
		&\bigl\langle (K_{i+1}-K_{i})-(K_{i}-K_{i-1}), K_{i+1}-K_i \bigr\rangle \nonumber\\
		&\quad =1/2 \cdot \|K_{i+1}-K_{i}\|_{\tf}^2
		-1/2 \cdot\|K_{i}-K_{i-1}\|_{\tf}^2
		+1/2 \cdot\|(K_{i+1}-K_{i})-(K_{i}-K_{i-1})\|_{\tf}^2 \nonumber\\
		&\quad \ge 1/2 \cdot\|K_{i+1}-K_{i}\|_{\tf}^2
		-1/2 \cdot\|K_{i}-K_{i-1}\|_{\tf}^2.
	\#
	Combining Lemmas \ref{sub1} and \ref{sub2}, we obtain
	\# \label{spe}
		&\bigl\langle (K_{i+1}-K_{i})-(K_{i}-K_{i-1}), K_{i+1}-K_i \bigr\rangle \\
		&\quad\le \eta \lambda\cdot\bigl(
		\tau_{V}^2 \cdot \|(K_{i+1}-K_i\|_{\tf}^2+\nu^2 \cdot \|\theta_{i}-\theta_{i-1}\|_{\tf}^2\bigr)
		-\eta/(2\lambda)\cdot\|\theta_{i+1}-\theta_i\|_{\tf}^2+\eta/(2\lambda)\cdot\|\theta_{i}-\theta_{i-1}\|_{\tf}^2  \nonumber \\
		&\quad\qquad -\eta\gamma\cdot \|\theta_i-\theta_{i-1}\|_{\tf}^2 + 2\eta \nu_V \sigma_{\theta}\cdot \| K_{i+1}-K_i \|_{\tf}^2+\eta \nu_V \sigma_{\theta}/2\cdot  \bigl(\|K_i-K_{i-1}\|_{\tf}^2+ \|K_{i+1}-K_i\|_{\tf}^2\bigr). \notag
	\#
	Further combining \eqref{pf9} and \eqref{spe}, we conclude the proof of Lemma \ref{alm:pf}.
	\end{proof}

\subsection{Proof of Lemma \ref{sub1}}\label{psub1}
\begin{proof}
	For the update of policy in \eqref{pg}, we have
	\#
	&(K_{i+1}-K_{i})-(K_{i}-K_{i-1}) \nonumber\\
	&\quad = -\eta\cdot \bigl(\nabla_K C(K_{i};\theta_i) - \nabla_K C(K_{i-1};\theta_{i-1})\bigr) \nonumber\\
	&\quad = -\eta \cdot {\textstyle\sum^{d+k}_{j,\ell}} \bigl(\nabla V_{j,\ell}(K_i) \cdot (\theta_i)_{j,\ell}   
	- \nabla V_{j,\ell}(K_{i-1}) \cdot (\theta_{i-1})_{j,\ell} \bigr) \nonumber\\
	&\quad= \underbrace{-\eta \cdot {\textstyle \sum^{d+k}_{j,\ell}} \nabla V_{j,\ell}(K_i) \cdot (\theta_i-\theta_{i-1})_{j,\ell} }_{\dr (i)}
	-\underbrace{\eta \cdot {\textstyle\sum^{d+k}_{j,\ell}} \bigl( \nabla V_{j,\ell}(K_{i})-\nabla V_{j,\ell}(K_{i-1})\bigr)\cdot (\theta_{i-1})_{j,\ell}}_{\dr (ii)} .\label{pf1}
	\#
	Hence, the inner product $\langle (K_{i+1}-K_{i})-(K_{i}-K_{i-1}), K_{i+1}-K_i \rangle$ is the difference between the inner products of $(K_{i+1}-K_i)$ and both terms (i) and (ii).
	
	In the sequel, we first establish an upper bound of the inner product $\langle K_{i+1}-K_i, (\text{i})\rangle$,
	\#
	&-\eta \cdot \bigl\langle  {\textstyle \sum^{d+k}_{j,\ell}} \nabla V_{j,\ell}(K_i) \cdot (\theta_i-\theta_{i-1})_{j,\ell},  K_{i+1}-K_i    \bigr\rangle  \nonumber \\
	&\quad = -\eta \cdot {\textstyle \sum^{d+k}_{j,\ell}}
	\bigl\langle 
	\nabla V_{j,\ell}(K_i), K_{i+1}-K_i \bigr\rangle \cdot (\theta_i-\theta_{i-1})_{j,\ell}  \nonumber\\
	&\quad = -\eta \cdot {\textstyle \sum^{d+k}_{j,\ell}} \Bigl(
	\bigl\langle 
	\nabla V_{j,\ell}(\tilde{K}_i^{j,\ell}), K_{i+1}-K_i \bigr\rangle 
	+\bigl \langle 
	\nabla V_{j,\ell}(K_i)-\nabla V_{j,\ell}(\tilde{K}_i^{j,\ell}), K_{i+1}-K_i \bigr \rangle \Bigr) \cdot (\theta_i-\theta_{i-1} )_{j,\ell} \nonumber \\
	&\quad = -\eta \cdot {\textstyle \sum^{d+k}_{j,\ell}} \Bigl(
	V_{j, \ell}(K_{i+1})-V_{j, \ell}(K_{i})+ \bigl \langle 
	\nabla V_{j,\ell}(K_i)-\nabla V_{j,\ell}(\tilde{K}_i^{j,\ell}), K_{i+1}-K_i  \bigr \rangle \Bigr)\cdot (\theta_i-\theta_{i-1})_{j,\ell} \nonumber \\
	&\quad \le \underbrace{-\eta \cdot \bigl\langle V(K_{i+1})-V(K_i),\theta_{i}-\theta_{i-1} \bigr\rangle}_{\dr (i).(a)}
	+ \underbrace{2\eta \nu_V\sigma_{\theta} \cdot \| K_{i+1}-K_i \|_{\tf}^2}_{\dr (i).(b)}, \label{pf2}
	\#
	where $\tilde{K}_i^{j,\ell}$ is the interpolation between $K_i$ and $K_{i+1}$ in the mean value theorem as defined in \eqref{ktilde}
	and the inequality follows from the same derivation of \eqref{cai812}, which is implied by the smoothness of $V(K)$ established in Lemma \ref{lm:lip1}. By \eqref{cg}, the first term (i).(a) in \eqref{pf2} takes the form
	\$
	&-\eta \cdot \bigl\langle V(K_{i+1})-V(K_i),\theta_{i}-\theta_{i-1} \bigr\rangle	 \\
	&\quad = -\eta \cdot \Bigl\langle
	1/\lambda\cdot\bigl((\theta_{i+1}-\theta_{i})-(\theta_{i}-\theta_{i-1})-(\varepsilon_{i+1}-\varepsilon_i)\bigr)+\bigl(\nabla \psi(\theta_i)- \nabla \psi(\theta_{i-1})\bigr),
	\theta_{i}-\theta_{i-1}
	\Bigr\rangle .
	\$
	Recall that as defined in \eqref{501}, $U_{i+1}=(\theta_{i+1}-\theta_{i})-(\theta_{i}-\theta_{i-1})$. Then we have
	\# 
	&-\eta \cdot \bigl\langle V(K_{i+1})-V(K_i),\theta_{i}-\theta_{i-1} \bigr\rangle \nonumber \\
	&\quad=
	-\eta/\lambda\cdot\bigl\langle 
	U_{i+1}-(\varepsilon_{i+1}-\varepsilon_i), \theta_{i}-\theta_{i-1}
	\bigr\rangle -\eta\cdot \bigl \langle 
	\nabla \psi(\theta_i)- \nabla \psi(\theta_{i-1}), \theta_{i}-\theta_{i-1}
	\bigr\rangle \label{pf3} \nonumber\\
	&\quad \le -\eta/\lambda\cdot\bigl\langle 
	U_{i+1}-(\varepsilon_{i+1}-\varepsilon_i), \theta_{i}-\theta_{i-1}
	\bigr\rangle -\eta  \gamma \cdot\|\theta_i-\theta_{i-1}\|_{\tf}^2.
	\#
	Here the inequality follows from the $\gamma$-strong convexity of $\psi(\cdot)$, which implies 
	\$ 
	\bigl \langle 
	\nabla \psi(\theta_i)- \nabla \psi(\theta_{i-1}), \theta_{i}-\theta_{i-1}
	\bigr\rangle \geq \gamma \cdot\|\theta_i-\theta_{i-1}\|_{\tf}^2.
	\$
	
	For the inner product $\langle (K_{i+1}-K_i),(\text{ii})\rangle$, invoking the smoothness of $V(K)$ established in Lemma \ref{lm:lip1} and the definition of $\sigma_\theta$ in \eqref{nota1}, we have
	\# \label{pf4}
	&\eta \cdot \Bigl \langle
	 {\textstyle \sum^{d+k}_{j,\ell}}  \bigl(\nabla V_{j,\ell}(K_{i})-\nabla V_{j,\ell}(K_{i-1})\bigr)\cdot (\theta_{i-1})_{j,\ell},K_{i+1}-K_i \Bigr\rangle \\
	&\quad \le
	\eta \nu_V \cdot\| \theta_{i-1} \|_{\tf}\cdot\|K_i-K_{i-1}\|_{\tf} \cdot\|K_{i+1}-K_i\|_{\tf}  \le \eta \nu_V \sigma_\theta/2\cdot \bigl(
	\|K_i-K_{i-1}\|_{\tf}^2+ \|K_{i+1}-K_i\|_{\tf}^2 
	\bigr),\notag
	\#
	where the first inequality follows from the same derivation of \eqref{cai812}.
	
	Combining \eqref{pf1}, \eqref{pf2}, \eqref{pf3}, and \eqref{pf4}, we obtain
	\$
	&\bigl\langle (K_{i+1}-K_{i})-(K_{i}-K_{i-1}), K_{i+1}-K_{i}  \bigr \rangle \nonumber	\\
	&\quad \le -\eta/\lambda\cdot\bigl\langle 
	(\theta_{i+1}-\theta_{i})-(\theta_{i}-\theta_{i-1})-(\varepsilon_{i+1}-\varepsilon_i), \theta_{i}-\theta_{i-1}
	\bigr\rangle -\eta  \gamma \cdot \|\theta_i-\theta_{i-1}\|_{\tf}^2 \nonumber  \\
	&
	\quad\qquad +2 \eta \nu_V \sigma_\theta \cdot \| K_{i+1}-K_i \|_{\tf}^2+\eta \nu_V \sigma_\theta/2\cdot \bigl(\|K_i-K_{i-1}\|_{\tf}^2+ \|K_{i+1}-K_i\|_{\tf}^2 \bigr), %\label{pf8}
	\$
	which concludes the proof of Lemma \ref{sub1}.
\end{proof}	
	
\subsection{Proof of Lemma \ref{sub2}}\label{psub2}
\begin{proof}
	By the definition of $\delta_{i+1}$ in \eqref{501}, we have
	\# \label{456}
	-(\theta_{i}-\theta_{i-1})=\delta_{i+1}-(\theta_{i+1}-\theta_{i}).
	\#
	Plugging \eqref{456} into the left-hand side of \eqref{500}, we obtain
	\# \label{652}
	&-\eta/\lambda\cdot\bigl\langle \delta_{i+1}-(\varepsilon_{i+1}-\varepsilon_i), \theta_{i}-\theta_{i-1}\bigr\rangle=
	\eta/\lambda\cdot\bigl\langle \delta_{i+1}-(\varepsilon_{i+1}-\varepsilon_i), \delta_{i+1}-(\theta_{i+1}-\theta_{i})\bigr\rangle \notag\\
	&\quad =\eta/\lambda\cdot\bigl\langle \delta_{i+1}-(\varepsilon_{i+1}-\varepsilon_i), \delta_{i+1} \bigr\rangle 
	- \eta/\lambda\cdot\bigl\langle \delta_{i+1},
	\theta_{i+1}-\theta_{i} \bigr\rangle
	+ \eta/\lambda\cdot
	\langle \varepsilon_{i+1}-\varepsilon_i, \theta_{i+1}-\theta_i  \rangle.
	\#
	By applying the identity
	\$
	\langle a,b \rangle=1/2\cdot \bigl(\|a\|^2+\|b\|^2-\|a-b\|^2\bigr)
	\$ to the first two terms on the right-hand side of \eqref{652}, we obtain
	\#
	&-\eta/\lambda\cdot\bigl\langle \delta_{i+1}-(\varepsilon_{i+1}-\varepsilon_i), \theta_{i}-\theta_{i-1} \bigr\rangle\nonumber\\
	&\quad =\eta/(2\lambda)\cdot \bigl(\|\delta_{i+1}-(\varepsilon_{i+1}-\varepsilon_i)\|_{\tf}^2+\|\delta_{i+1}\|_{\tf}^2-\|\varepsilon_{i+1}-\varepsilon_i\|_{\tf}^2\bigr)\nonumber\\
	&\quad \qquad- \eta/(2\lambda)\cdot\bigl(\|\theta_{i+1}-\theta_i\|_{\tf}^2-\|\theta_{i}-\theta_{i-1}\|_{\tf}^2 
	+\|\delta_{i+1}\|_{\tf}^2\bigr) + \eta/\lambda\cdot
	\langle \varepsilon_{i+1}-\varepsilon_i, \theta_{i+1}-\theta_i  \rangle. \label{pf5}
	\#
	By rearranging the terms in \eqref{pf5} and invoking the facts $\langle \varepsilon_{i+1}, \theta_{i+1}-\theta_i  \rangle\le0$ and $\langle \varepsilon_i, \theta_i - \theta_{i+1} \rangle\le0$, which follows from the convexity of $\Theta$ and the fact that 
	\$
	\theta_{i+1} = \Pi_\Theta\bigl[\theta_i+\lambda\cdot \nabla_\theta m(K_{i+1}, \theta_i)\bigr], \quad \theta_{i} = \Pi_\Theta\bigl[\theta_{i-1}+\lambda\cdot \nabla_\theta m(K_{i}, \theta_{i-1})\bigr],
	\$ we obtain
	\# \label{pf6}
	&-\eta/\lambda\cdot\bigl\langle \delta_{i+1}-(\varepsilon_{i+1}-\varepsilon_i), \theta_{i}-\theta_{i-1} \bigr\rangle \notag\\
	&\quad\le \eta/(2\lambda)\cdot\|\delta_{i+1}-(\varepsilon_{i+1}-\varepsilon_i)\|_{\tf}^2- \eta/(2\lambda)\cdot\|\theta_{i+1}-\theta_i\|_{\tf}^2
	+ \eta/(2\lambda)\cdot\|\theta_{i}-\theta_{i-1}\|_{\tf}^2 .
	\#
	By the definition of $\delta_{i+1}$ in \eqref{501}, we have
	\$
	\delta_{i+1}-(\varepsilon_{i+1}-\varepsilon_i)
	&=
	( \theta_{i+1}-\theta_i-\varepsilon_{i+1} )
	-( \theta_{i}-\theta_{i-1}-\varepsilon_{i} )
	\\
	&=\lambda\cdot \Bigl( \bigl(V(K_{i+1})-V(K_i)\bigr)-\bigl(\nabla \psi(\theta_i)-\nabla \psi(\theta_{i-1})\bigr) \Bigr).
	\$
	Using the Cauchy-Schwarz inequality, we obtain
	\#\label{pf7}
	\eta/(2\lambda)\cdot\|\delta_{i+1}-(\varepsilon_{i+1}-\varepsilon_i)\|_{\tf}^2 
	&\le \eta\lambda\cdot \bigl( \|V(K_{i+1})-V(K_i)\|_{\tf}^2+\|\nabla \psi(\theta_i)-\nabla \psi(\theta_{i-1})\|_{\tf}^2 \bigr) \nonumber\\
	&\le \eta\lambda\cdot \bigl(
	\tau_{V}^2 \cdot \|K_{i+1}-K_i\|_{\tf}^2+\nu^2 \cdot \|\theta_{i}-\theta_{i-1}\|_{\tf}^2\bigr),
	\#
	where the second inequality follows from the smoothness of $V(K)$ established in Lemma \ref{lm:lip1} and the $\nu$-smoothness of $\psi(\cdot)$. 
	
	Plugging \eqref{pf7} into \eqref{pf6} yields \eqref{500}. Thus, we conclude the proof of Lemma \ref{sub2}.
	\end{proof}		
	
	\section{Auxiliary Geometric Lemmas}\label{sec:aux}
	For completeness, we present several lemmas in \cite{fazel2018global} that characterize the geometry of the cost function $C(K;\theta)$ with respect to the policy $K$. 
	
	\begin{lemma}[Policy Gradient, Lemma 1 in \cite{fazel2018global}] \label{pgl}
	It holds that
	\$
	\nabla_K C(K;\theta)=2\bigl( (R+B^\top P_KB)K-B^\top P_KA  \bigr)\Sigma_K,
	\$
	where $P_K$ is defined in \eqref{eq:wa6}.
	\end{lemma}

	For notational simplicity, we denote 
	\$
	E_K=(R+B^\top P_KB)K-B^\top P_KA.
	\$
	 We have $\nabla_K C(K;\theta)=2E_K\Sigma_K$ from Lemma \ref{pgl}. Note that $\nabla_K C(K;\theta)=0$ if and only if $E_K=0$, since we assume that $\Sigma_0$ is positive definite, which implies that $\Sigma_K$ is positive definite.
	
	\begin{lemma}[Difference of Cost, Lemma 12 in \cite{fazel2018global}]
	We assume that $K$ and $K'$ are stabilizing policies. The cost function $C(K;\theta)$ satisfies
	\$
	C(K';\theta)-C(K;\theta)=-2\tr\bigl(  \Sigma_{K'}(K-K')^\top E_K  \bigr)
	+\tr\bigl( \Sigma_{K'}(K-K')^\top (R+B^\top P_K B)(K-K')  \bigr).
	\$
	Specifically, if $\nabla_K C(K;\theta)=0$, which implies $E_K=0$, then $C(K;\theta)$ satisfies
	\$
	C(K';\theta)-C(K;\theta)=\tr\bigl( \Sigma_{K'}(K-K')^\top (R+B^\top P_K B)(K-K')  \bigr).
	\$
	\end{lemma}
	
	\begin{lemma}[Upper Bound of  Policy Gradient, Lemma 22 in \cite{fazel2018global}]
	It holds that
	\$
	\|\nabla C(K;\theta)\|\le \frac{C(K;\theta)}{\mu^{1/2} \cdot \sigma_{\text{min}}(Q)} \cdot 
	\biggl(
	\|R+B^\top P_K B \| \cdot \Bigl(C(K;\theta)-C\bigl(K^*(\theta);\theta\bigr)\Bigr)\biggr)^{1/2},
	\$
	where $K^*(\theta)$ is defined in \eqref{kstar}.
	\end{lemma}

\end{appendix}

\end{document}